\documentclass{article}

\usepackage{microtype}
\usepackage{graphicx}
\usepackage{subfigure}
\usepackage{booktabs} 
\usepackage{multirow}
\usepackage{amsthm}
\usepackage{amsfonts}
\usepackage{amsmath}
\usepackage{mathrsfs}
\usepackage{empheq}
\usepackage{todonotes}
\usepackage{siunitx}
\newcommand*\widefbox[1]{\fbox{\hspace{1em}#1\hspace{1em}}}

\usepackage{hyperref}

\usepackage[accepted]{icml2019}

\icmltitlerunning{Minimal Achievable Sufficient Statistic Learning}

\newcommand{\R}{\mathbb{R}}
\newcommand{\Prob}{\mathbb{P}}

\newcommand{\E}{\mathbb{E}}

\newcommand{\Tr}{\text{Tr}}
\def\intd{\mathrm{d}}
\providecommand{\Hm}[1]{\mathscr{H}^{#1}}
\def\trans{{\operatorname{T}}}

\newcommand{\be}{\begin{equation}}
\newcommand{\ee}{\end{equation}}
\def\ba#1\ea{\begin{align}#1\end{align}}

\theoremstyle{definition}
\newtheorem{definition}{Definition}
\theoremstyle{plain}
\newtheorem{theorem}{Theorem}
\newtheorem{lemma}{Lemma}

\begin{document}

\twocolumn[
\icmltitle{Minimal Achievable Sufficient Statistic Learning}



\icmlsetsymbol{equal}{*}

\begin{icmlauthorlist}
\icmlauthor{Milan Cvitkovic}{caltech}
\icmlauthor{G\"unther Koliander}{ari}
\end{icmlauthorlist}

\icmlaffiliation{caltech}{Department of Computing and Mathematical Sciences, California Institute of Technology, Pasadena, California, USA}
\icmlaffiliation{ari}{Acoustics Research Institute, Austrian Academy of Sciences, Vienna, Austria}

\icmlcorrespondingauthor{Milan Cvitkovic}{mcvitkov@caltech.edu}

\icmlkeywords{Machine Learning, Representation Learning, Deep Learning, Minimal Sufficient Statistics, Information Bottleneck, Singular Distributions, Adversarial Examples, Uncertainty Quantification}

\vskip 0.3in
]



\printAffiliationsAndNotice{}  

\begin{abstract}
We introduce Minimal Achievable Sufficient Statistic (MASS) Learning, a training method for machine learning models that attempts to produce minimal sufficient statistics with respect to a class of functions (e.g. deep networks) being optimized over. In deriving MASS Learning, we also introduce Conserved Differential Information (CDI), an information-theoretic quantity that --- unlike standard mutual information --- can be usefully applied to deterministically-dependent continuous random variables like the input and output of a deep network. In a series of experiments, we show that deep networks trained with MASS Learning achieve competitive performance on supervised learning and uncertainty quantification benchmarks.
\end{abstract}

\section{Introduction}
\label{sec:Introduction}
The \emph{representation learning} approach to machine learning focuses on finding a representation $Z$ of an input random variable $X$ that is useful for predicting a random variable $Y$ \cite{goodfellow_deep_2016}.

What makes a representation $Z$ ``useful'' is  much debated, but a common assertion is that $Z$ should be a \emph{minimal sufficient statistic} of $X$ for $Y$\cite{adragni_kofi_p._sufficient_2009, shamir_learning_2010, james_trimming_2017,  achille_information_2018}. That is:
 
\begin{enumerate}
    \item $Z$ should be a \emph{statistic} of $X$.  This means $Z = f(X)$ for some function $f$.
    \item $Z$ should be \emph{sufficient} for $Y$. This means $p(X|Z,Y) = p(X|Z)$.
    \item Given that $Z$ is a sufficient statistic, it should be \emph{minimal} with respect to $X$. This means for any measurable, non-invertible function $g$, $g(Z)$ is no longer sufficient for $Y$.\footnote{This is not the most common phrasing of statistical minimality, but we feel it is more understandable.  For the equivalence of this phrasing and the standard definition see Supplementary Material \ref{supp:MSS}.}
\end{enumerate}

In other words: a minimal sufficient statistic is a random variable $Z$ that tells you everything about $Y$ you could ever care about, but if you do any irreversible processing to $Z$, you are guaranteed to lose some information about $Y$.

Minimal sufficient statistics have a long history in the field of statistics \cite{lehmann_completeness_1950, dynkin_necessary_1951}. But the minimality condition (3, above) is perhaps too strong to be useful in machine learning, since it is a statement about \emph{any} function $g$, rather than about functions in a practical hypothesis class like the class of deep neural networks.

Instead, in this work we consider \emph{minimal achievable sufficient statistics}: sufficient statistics that are minimal among some particular set of functions.

\begin{definition}[Minimal Achievable Sufficient Statistic]
\label{def:MASS}
Let $Z = f(X)$ be a sufficient statistic of $X$ for $Y$.  $Z$ is \emph{minimal achievable} with respect to a set of functions $\mathcal{F}$ if $f \in \mathcal{F}$ and for any Lipschitz continuous, non-invertible function $g$ where $g \circ f \in \mathcal{F}$, $g(Z)$ is no longer sufficient for $Y$.
\end{definition}

\paragraph{Contributions:}
\begin{itemize}
    \item We introduce Conserved Differential Information (CDI), an information-theoretic quantity that, unlike mutual information, is meaningful for deterministically-dependent continuous random variables, such as the input and output of a deep network.
    \item We introduce Minimal Achievable Sufficient Statistic Learning (MASS Learning), a training objective based on CDI for finding minimal achievable sufficient statistics.
    \item We provide empirical evidence that models trained by MASS Learning achieve competitive performance on supervised learning and uncertainty quantification benchmarks.
\end{itemize}

\section{Conserved Differential Information}

Before we present MASS Learning, we need to introduce Conserved Differential Information (CDI), on which MASS Learning is based.

CDI is an information-theoretic quantity that addresses an oft-cited issue in machine learning \cite{bell_information-maximization_1995, amjad_learning_2018,saxe_information_2018, nash_inverting_2018,goldfeld_estimating_2018}, which is that for a continuous random variable $X$ and a continuous, non-constant function $f$, the mutual information $I(X, f(X))$ is infinite. (See Supplementary Material \ref{supp:MIInfinite} for details.)  This makes  $I(X, f(X))$ unsuitable for use in a learning objective when $f$ is, for example, a standard deep network.

The infinitude of $I(X, f(X))$ has been circumvented in prior works by two strategies.  One is discretize $X$ and $f(X)$ \cite{tishby_deep_2015, shwartz-ziv_opening_2017}, though this is controversial \cite{saxe_information_2018}.  Another is to use a random variable $Z$ with distribution $p(Z\vert X)$ as the representation of $X$ rather than using $f(X)$ itself as the representation \cite{alemi_deep_2016, kolchinsky_nonlinear_2017, achille_information_2018}.  In this latter approach, $p(Z\vert X)$ is usually implemented by adding noise to a deep network that takes $X$ as input.

These are both reasonable strategies for avoiding the infinitude of $I(X, f(X))$.  But another approach would be to derive a new information-theoretic quantity that is better suited to this situation.  To that end we present Conserved Differential Information:

\begin{definition} For a continuous random variable $X$ taking values in $\R^d$ and a Lipschitz continuous function $f \colon \R^d \to \R^r$, the \textbf{Conserved Differential Information} (CDI) is
\begin{empheq}[box=\fbox]{align}
C(X, f(X)) := H(f(X)) -  \E_X\left[\log \left( J_f(X) \right)\right]
\end{empheq}
where $H$ denotes  the differential entropy 
\[
H(Z) = - \int p(z) \log p(z) \, \mathrm{d}z 
\] 
and $J_f$ is the Jacobian determinant of $f$ 
\[
J_f(x) = \sqrt{\det \left(\frac{\partial f(x)}{\partial x^\trans}\left(\frac{\partial f(x)}{\partial x^\trans}\right)^\trans\right)}
\] 
with $\frac{\partial f(x)}{\partial x^\trans} \in \R^{r \times d}$ the Jacobian matrix of $f$ at $x$.
\end{definition}

Readers familiar with normalizing flows \cite{rezende_variational_2015} or Real NVP \cite{dinh_density_2016} will note that the Jacobian determinant used in those methods is a special case of the Jacobian determinant in the definition of CDI. 
This is because normalizing flows and Real NVP are based on the change of variables formula for invertible mappings, while CDI is based in part on the more general change of variables formula for non-invertible mappings.  
More details on this connection are given in Supplementary Material  \ref{supp:coarea}.
The mathematical motivation for CDI based on the recent work of Koliander et al. \yrcite{koliander_entropy_2016} is provided in Supplementary Material \ref{supp:CDIDerivation}.  
Figure \ref{fig:CDIExample} gives a visual example of what CDI measures about a function.

\begin{figure}[ht]
\centerline{\includegraphics[width=0.5\textwidth]{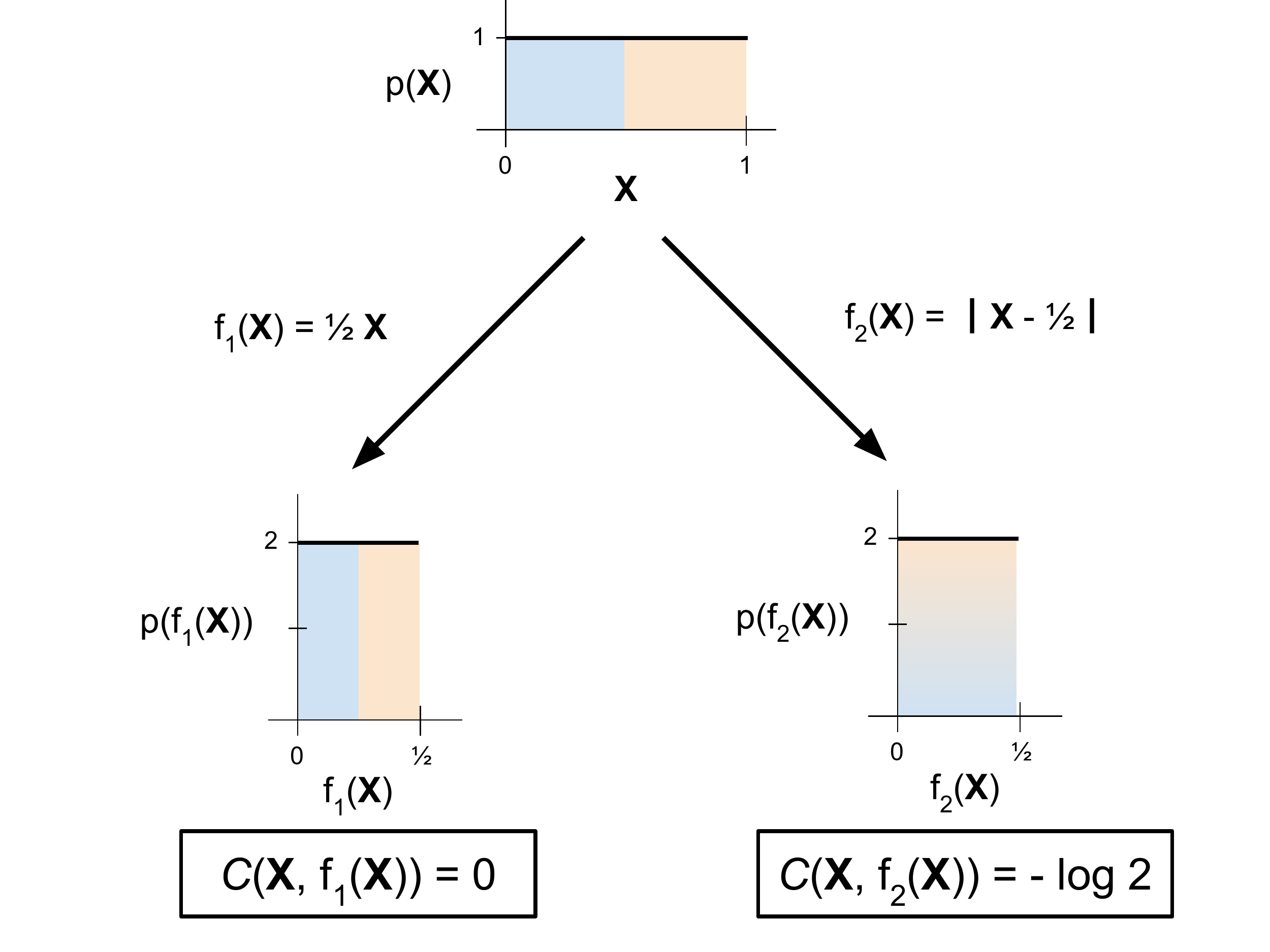}}
\caption{CDI of two functions $f_1$ and $f_2$ of the random variable $X$.  Even though the random variables $f_1(X)$ and $f_2(X)$ have the same distribution, $C(X, f_1(X))$ is different from $C(X, f_2(X))$. This is because $f_1$ is an invertible function, while $f_2$ is not.  CDI quantifies, roughly speaking, ``how non-invertible'' $f_2$ is.}
\label{fig:CDIExample}
\end{figure}

The conserved differential information $C(X,f(X))$ between continuous, deterministically-dependent random variables behaves much like mutual information does between discrete random variables.  For example, when $f$ is invertible, $C(X,f(X)) = H(X)$, just like with the mutual information between discrete random variables.  
Most importantly for our purposes, though, $C(X, f(X))$ obeys the following data processing inequality:

\begin{theorem}[CDI Data Processing Inequality]
\label{thm:dpi}
For Lipschitz continuous functions $f$ and $g$ with the same output space,
\[
C\left(X, f(X)\right) \geq C\left(X, g(f(X))\right)
\]
with equality if and only if $g$ is invertible almost everywhere.
\end{theorem}
The proof is in Supplementary Material \ref{supp:CDIProof}.

\section{MASS Learning}
With CDI and its data processing inequality in hand, we can give the following optimization-based characterization of minimal achievable sufficient statistics:

\begin{theorem}
\label{thm:MASS}
Let $X$ be a continuous random variable, $Y$ be a discrete random variable, and $\mathcal{F}$ be any set of Lipschitz continuous functions with a common output space (e.g., different parameter settings of a deep network).
If 
\begin{align*}
f \in &\arg\min_{S \in \mathcal{F}}  \ C(X,S(X)) \\
&s.t. \ \ I(S(X), Y) = \max_{S'} I(S'(X), Y)
\end{align*}
then 
$f(X)$ is a minimal achievable sufficient statistic of $X$ for $Y$ with respect to $\mathcal{F}$. 
\end{theorem}

\begin{proof}
First note the following lemma \cite{cover_elements_2006}.
\begin{lemma}
\label{lem:SuffStat}
$Z = f(X)$ is a sufficient statistic for a discrete random variable $Y$ if and only if $I(Z, Y) = \max_{S'} I(S'(X), Y)$.
\end{lemma}
Lemma~\ref{lem:SuffStat} guarantees that any $f$ satisfying the conditions in Theorem~\ref{thm:MASS} is sufficient. Suppose such an $f$ was not minimal achievable.  Then by Definition \ref{def:MASS} there would exist a non-invertible, Lipschitz continuous $g$ such that $g(f(X))$ was sufficient.  But by Theorem~\ref{thm:dpi}, it would then also be the case that $C(X,g(f(X))) <  C(X,f(X))$, which would contradict $f$ minimizing $C(X,S(X))$.
\end{proof}

We can turn Theorem \ref{thm:MASS} into a learning objective over functions $f$ by relaxing the strict constraint into a Lagrangian formulation with Lagrange multiplier $1/\beta$ for $\beta > 0$:
\begin{align*}
C(X,f(X)) -  \frac{1}{\beta}I(f(X),Y)
\end{align*}
 The larger the value of $\beta$, the more our objective will encourage minimality over sufficiency. 
We can then simplify this formulation using the identity $I(f(X),Y) = H(Y) - H(Y|f(X))$, which gives us the following optimization objective:
\begin{empheq}[box=\widefbox]{align}
\begin{split}
\mathcal{L}_{MASS}(f) & := H(Y|f(X)) + \beta H(f(X)) \\
& \quad - \beta \E_X[\log J_{f}(X)].
\end{split}
\end{empheq}
We refer to minimizing this objective as \textbf{MASS Learning}.

\subsection{Practical implementation}

In practice, we are interested in using MASS Learning to train a deep network $f_\theta$ with parameters $\theta$ using a finite dataset $\{(x_i,y_i)\}_{i=1}^N $ of $N$ datapoints sampled from the joint distribution $p(x,y)$ of $X$ and $Y$.  To do this, we introduce a parameterized variational approximation $q_\phi(f_\theta(x) | y) \approx p(f_{\theta}(x)|y)$.  Using $q_\phi$, we minimize the following empirical upper bound to $\mathcal{L}_{MASS}$:
\begin{align*}
\mathcal{L}_{MASS} \leq \widehat{\mathcal{L}}_{MASS}(\theta, \phi) := \frac{1}{N} & \sum_{i=1}^N - \log q_\phi(y_i|f_\theta(x_i)) \\
& - \beta \log q_\phi(f_\theta(x_i)) \\
& -  \beta  \log J_{f_\theta}(x_i),
\end{align*}
where the quantity $q_\phi(f_\theta(x_i))$ is computed as $\sum_y q_\phi(f_\theta(x_i)|y)p(y)$ and the quantity $q_\phi(y_i|f_\theta(x_i))$ is computed with Bayes rule as $\frac{q_\phi(f_\theta(x_i)|y_i)p(y_i)}{\sum_y q_\phi(f_\theta(x_i)|y)p(y)}$.
When $Y$ is discrete and takes on finitely many values, as in classification problems, and when we choose a variational distribution $q_\phi$ that is differentiable with respect to $\phi$ (e.g. a multivariate Gaussian), then we can minimize $\widehat{\mathcal{L}}_{MASS}(\theta, \phi)$ using stochastic gradient descent (SGD).

To perform classification using our trained network, we use the learned variational distribution $q_\phi$ and Bayes rule:
\begin{align*}
p(y_i|x_i) \approx p(y_i|f_\theta(x_i)) \approx \frac{q_\phi(f_\theta(x_i)|y_i)p(y_i)}{\sum_y q_\phi(f_\theta(x_i)|y)p(y)}.
\end{align*}

Computing the $J_{f_\theta}$ term in $\widehat{\mathcal{L}}_{MASS}$ for every sample in an SGD minibatch is too expensive to be practical.  For $f_\theta \colon \R^d \to \R^r$, doing so would require on the order of $r$ times more operations than in standard training of deep networks by, since computing the $J_{f_\theta}$ term involves computing the full Jacobian matrix of the network, which, in our implementation, involves performing $r$ backpropagations.  Thus to make training tractable, we use a subsampling strategy: we estimate the $J_{f_\theta}$ term using only a $1/r$ fraction of the datapoints in a minibatch.  In practice, we have found this subsampling strategy to not noticeably alter the numerical value of the $J_{f_\theta}$ term during training.  

Subsampling for the $J_{f_\theta}$ term results in a significant training speedup, but it must nevertheless be emphasized that, even with subsampling, our implementation of MASS Learning is roughly eight times as slow as standard deep network training.  (Unless $\beta=0$, in which case the speed is the same.)  This is by far the most significant drawback of (our implementation of) MASS Learning.  There are many easier-to-compute upper bounds or estimates of $J_{f_\theta}$ that one could use to make MASS Learning faster, and one could also potentially find non-invertible network architectures which admit more efficiently computable Jacobians, but we do not explore these options in this work.

\section{Related Work}

\subsection{Connection to the Information Bottleneck}
The well-studied Information Bottleneck learning method \cite{tishby_information_2000, tishby_deep_2015, strouse_deterministic_2016, alemi_deep_2016, saxe_information_2018, amjad_learning_2018, goldfeld_estimating_2018, kolchinsky_caveats_2018, achille_information_2018, achille_emergence_2018} is based on minimizing the Information Bottleneck Lagrangian
\begin{equation*}
    \mathcal{L}_{IB}(Z) := \beta I(X,Z) - I(Y,Z)
\end{equation*} for $\beta > 0$, where $Z$ is the representation whose conditional distribution $p(Z|X)$ one is trying to learn.

The $\mathcal{L}_{IB}$ learning objective can be motivated based on pure information-theoretic elegance.  But some works like \cite{shamir_learning_2010} also point out the connection between the $\mathcal{L}_{IB}$ objective and minimal sufficient statistics, which is based on the following theorem:
\begin{theorem}\label{thm:DiscreteIBMSS}
Let $X$ be a discrete random variable drawn according to a distribution $p(X|Y)$ determined by the discrete random variable $Y$.  Let $\mathcal{F}$ be the set of deterministic functions of $X$ to any target space.  Then $f(X)$ is a minimal sufficient statistic of $X$ for $Y$ if and only if
\begin{align*}
f \in &\arg\min_{S \in \mathcal{F}}  \ I(X,S(X)) \\
&s.t. \ \ I(S(X), Y) = \max_{S' \in \mathcal{F}} I(S'(X), Y).
\end{align*}
\end{theorem}
The $\mathcal{L}_{IB}$ objective can then be thought of as a Lagrangian relaxation of the optimization problem in this theorem.

Theorem \ref{thm:DiscreteIBMSS} only holds for discrete random variables.  For continuous $X$ it holds only in the reverse direction, so minimizing $\mathcal{L}_{IB}$ for continuous $X$ has no formal connection to finding minimal sufficient statistics, not to mention minimal achievable sufficient statistics.  See Supplementary Material \ref{supp:IBReverseDirection} for details.

Nevertheless, the optimization problems in Theorem \ref{thm:MASS} and Theorem \ref{thm:DiscreteIBMSS} are extremely similar, relying as they both do on Lemma \ref{lem:SuffStat} for their proofs.  And the idea of relaxing the optimization problem in Theorem \ref{thm:MASS} into a Lagrangian formulation to get $\mathcal{L}_{MASS}$ is directly inspired by the Information Bottleneck.  So while MASS Learning and Information Bottleneck learning entail different network architectures and loss functions, there is an Information Bottleneck flavor to MASS Learning.

\subsection{Jacobian Regularization}

The presence of the $J_{f_\theta}$ term in $\widehat{\mathcal{L}}_{MASS}$ is reminiscent of the contrastive autoencoder \cite{rifai_higher_2011} and Jacobian Regularization literature \cite{sokolic_robust_2017, ross_improving_2017, varga_gradient_2017, novak_sensitivity_2018, jakubovitz_improving_2018}.  Both these literatures suggest that minimizing $\E_X[\Vert D_f(X) \Vert_F]$, where $D_f(x) = \frac{\partial f(x)}{\partial x^\trans} \in  \R^{r\times d}$ is the Jacobian matrix, seems to improve generalization and adversarial robustness.  

This may seem paradoxical at first, since by applying the AM-GM inequality to the eigenvalues of $D_f(x)D_f(x)^\trans$ we have
\begin{align*}
\E_X[\Vert D_f(X) \Vert_F^{2r}] &=
\E_X[\Tr(D_f(X)D_f(X)^\trans)^{r}] \\ &\geq 
\E_X[r^r\det(D_f(X)D_f(X)^\trans)] \\ &= 
\E_X[r^r J_f(X)^2] \\ &\geq 
\log \E_X[r^r J_f(X)^2] \\ &\geq 
2 \E_X[\log J_f(X)] +r \log r
\end{align*}
and $\E_X[\log J_f(X)]$ is being \emph{maximized} by $\widehat{\mathcal{L}}_{MASS}$.  So $\widehat{\mathcal{L}}_{MASS}$ might seem to be optimizing for worse generalization according to the Jacobian regularization literature.  However, the entropy term in $\widehat{\mathcal{L}}_{MASS}$ strongly encourages minimizing $\E_X[\Vert D_f(X) \Vert_F]$.  So overall $\widehat{\mathcal{L}}_{MASS}$ seems to be seeking the right balance of sensitivity (dependent on the value of $\beta$) in the network to its inputs, which is precisely in alignment with what the Jacobian regularization literature suggests.

\section{Experiments}

In this section we compare MASS Learning to other approaches for training deep networks.  Code to reproduce all experiments is available online.\footnote{\url{https://github.com/mwcvitkovic/MASS-Learning}}  Full details on all experiments is in Supplementary Material \ref{supp:ExperimentDetails}.

We use the abbreviation ``SoftmaxCE'' to refer to the standard approach of training deep networks for classification problems by minimizing the softmax cross entropy loss
\[
\widehat{\mathcal{L}}_{SoftmaxCE}(\theta) := - \frac{1}{N} \sum_{i=1}^N \Big( \log \texttt{softmax}(f_\theta(x_i))_{y_i} \Big)
\]
where $\texttt{softmax}(f_\theta(x_i))_{y_i}$ is the $y_i$th element of the softmax function applied to the outputs $f_\theta(x_i)$ of the network's last linear layer.  As usual, $\texttt{softmax}(f_\theta(x_i))_{y_i}$ is taken to be the network's estimate of $p(y_i|x_i)$.

We also compare against the Variational Information Bottleneck method \cite{alemi_deep_2016} for representation learning, which we abbreviate as ``VIB''.

We use two networks in our experiments.  ``SmallMLP'' is a feedforward network with two fully-connected layers of 400 and 200 hidden units, respectively, both with \texttt{elu} nonlinearities \cite{clevert_fast_2015}.  ``ResNet20'' is the 20-layer residual network of He et al. \yrcite{he_deep_2015}.

We performed all experiments on the CIFAR-10 dataset \cite{Krizhevsky2009LearningML} and implemented all experiments using PyTorch \cite{paszke_automatic_2017}.

\subsection{Classification Accuracy and Regularization}

We first confirm that networks trained by MASS Learning can make accurate predictions in supervised learning tasks.  We compare the classification accuracy of networks trained on varying amounts of data to see the extent to which MASS Learning regularizes networks.

Classification accuracies for the SmallMLP network are shown in Table \ref{tab:CIFAR-10AccuracySmallMLP}, and for the ResNet20 network in Table \ref{tab:CIFAR-10AccuracyResNet20}.  For the SmallMLP network, MASS Learning performs slightly worse than SoftmaxCE and VIB training.  For the larger ResNet20 network, MASS Learning performs equivalently to the other methods.  It is notable that with the ResNet20 network VIB and MASS Learning both perform well when $\beta = 0$, and neither perform significantly better than SoftmaxCE.  This may be because the hyperparameters used in training the ResNet20 network, which were taken directly from the original paper \cite{he_deep_2015}, are specifically tuned for SoftmaxCE training and are more sensitive to the specifics of the network architecture  than to the loss function.

\begin{table}
\caption{Test-set classification accuracy (percent) on CIFAR-10 dataset using the SmallMLP network trained by various methods.  Full experiment details are in Supplementary Material \ref{supp:ExperimentDetails}.  Values are the mean classification accuracy over 4 training runs with different random seeds, plus or minus the standard deviation. Emboldened accuracies are those for which the maximum observed mean accuracy in the column was within one standard deviation. WD is weight decay; D is dropout.}
\label{tab:CIFAR-10AccuracySmallMLP}
\vskip 0.15in
\begin{center}
\begin{small}
\begin{tabular}{@{}l|ccc@{}}
\toprule
\multicolumn{1}{c|}{\multirow{2}{*}{\sc{Method}}} & \multicolumn{3}{c}{\sc{Training Set Size}} \\
\multicolumn{1}{c|}{} & \multicolumn{1}{c}{\sc{2500}} & \sc{10,000} & \sc{40,000} \\ \midrule

SoftmaxCE	&   $ 34.2 \pm 0.8 $	&   $ \mathbf{44.6 \pm 0.6} $	&   $ 52.7 \pm 0.4 $	\\
SoftmaxCE, WD	&   $ 23.9 \pm 0.9 $	&   $ 36.4 \pm 0.9 $	&   $ 48.1 \pm 0.1 $	\\
SoftmaxCE, D	&   $ 33.7 \pm 1.1 $	&   $ 44.1 \pm 0.6 $	&   $ 53.7 \pm 0.3 $	\\
VIB, $\beta$=\num{1e-1}	&   $ 32.2 \pm 0.6 $	&   $ 40.6 \pm 0.4 $	&   $ 46.1 \pm 0.5 $	\\
VIB, $\beta$=\num{1e-2}	&   $ 34.6 \pm 0.4 $	&   $ 43.8 \pm 0.8 $	&   $ 51.9 \pm 0.8 $	\\
VIB, $\beta$=\num{1e-3}	&   $ \mathbf{35.6 \pm 0.5} $	&   $ \mathbf{44.6 \pm 0.6} $	&   $ 51.8 \pm 0.8 $	\\
VIB, $\beta$=\num{1e-1}, D	&   $ 29.0 \pm 0.6 $	&   $ 40.1 \pm 0.5 $	&   $ 49.5 \pm 0.5 $	\\
VIB, $\beta$=\num{1e-2}, D	&   $ 32.5 \pm 0.9 $	&   $ 43.9 \pm 0.3 $	&   $ 53.6 \pm 0.3 $	\\
VIB, $\beta$=\num{1e-3}, D	&   $ 34.5 \pm 1.0 $	&   $ \mathbf{44.4 \pm 0.4} $	&   $ \mathbf{54.3 \pm 0.2} $	\\
MASS, $\beta$=\num{1e-2}	&   $ 29.6 \pm 0.4 $	&   $ 39.9 \pm 1.2 $	&   $ 46.3 \pm 1.2 $	\\
MASS, $\beta$=\num{1e-3}	&   $ 32.7 \pm 0.8 $	&   $ 41.5 \pm 0.7 $	&   $ 47.8 \pm 0.8 $	\\
MASS, $\beta$=\num{1e-4}	&   $ 34.0 \pm 0.3 $	&   $ 41.5 \pm 1.1 $	&   $ 47.9 \pm 0.8 $	\\
MASS, $\beta$=\num{0}	&   $ 34.1 \pm 0.6 $	&   $ 42.0 \pm 0.6 $	&   $ 48.2 \pm 0.9 $	\\
MASS, $\beta$=\num{1e-2}, D	&   $ 29.3 \pm 1.2 $	&   $ 41.7 \pm 0.4 $	&   $ 52.0 \pm 0.6 $	\\
MASS, $\beta$=\num{1e-3}, D	&   $ 31.5 \pm 0.6 $	&   $ 43.7 \pm 0.2 $	&   $ 53.1 \pm 0.4 $	\\
MASS, $\beta$=\num{1e-4}, D	&   $ 32.7 \pm 0.8 $	&   $ 43.4 \pm 0.5 $	&   $ 53.2 \pm 0.1 $	\\
MASS, $\beta$=\num{0}, D	&   $ 32.2 \pm 1.1 $	&   $ 43.9 \pm 0.4 $	&   $ 52.7 \pm 0.0 $	\\

\bottomrule
\end{tabular}
\end{small}
\end{center}
\vskip -0.1in
\end{table}

\begin{table}
\caption{Test-set classification accuracy (percent) on CIFAR-10 dataset using the ResNet20 network trained by various methods.  No data augmentation was used --- full details in Supplementary Material \ref{supp:ExperimentDetails}.  Values are the mean classification accuracy over 4 training runs with different random seeds, plus or minus the standard deviation. Emboldened accuracies are those for which the maximum observed mean accuracy in the column was within one standard deviation.}
\label{tab:CIFAR-10AccuracyResNet20}
\vskip 0.15in
\begin{center}
\begin{small}
\begin{tabular}{@{}l|ccc@{}}
\toprule
\multicolumn{1}{c|}{\multirow{2}{*}{\sc{Method}}} & \multicolumn{3}{c}{\sc{Training Set Size}} \\
\multicolumn{1}{c|}{} & \multicolumn{1}{c}{\sc{2500}} & \sc{10,000} & \sc{40,000} \\ \midrule

SoftmaxCE	&   $ \mathbf{50.0 \pm 0.7} $	&   $ \mathbf{67.5 \pm 0.8} $	&   $ \mathbf{81.7 \pm 0.3} $	\\
VIB, $\beta$=\num{1e-3}	&   $ \mathbf{49.5 \pm 1.1} $	&   $ \mathbf{66.9 \pm 1.0} $	&   $ 81.0 \pm 0.3 $	\\
VIB, $\beta$=\num{1e-4}	&   $ 49.4 \pm 1.0 $	&   $ 66.4 \pm 0.5 $	&   $ 81.2 \pm 0.4 $	\\
VIB, $\beta$=\num{1e-5}	&   $ \mathbf{50.0 \pm 1.1} $	&   $ \mathbf{67.9 \pm 0.8} $	&   $ 80.9 \pm 0.5 $	\\
VIB, $\beta$=\num{0}	&   $ \mathbf{50.6 \pm 0.8} $	&   $ \mathbf{67.1 \pm 1.0} $	&   $ \mathbf{81.5 \pm 0.2} $	\\
MASS, $\beta$=\num{1e-3}	&   $ 38.2 \pm 0.7 $	&   $ 59.6 \pm 0.8 $	&   $ 75.8 \pm 0.5 $	\\
MASS, $\beta$=\num{1e-4}	&   $ \mathbf{49.9 \pm 1.0} $	&   $ 66.6 \pm 0.4 $	&   $ 80.6 \pm 0.5 $	\\
MASS, $\beta$=\num{1e-5}	&   $ \mathbf{50.1 \pm 0.5} $	&   $ \mathbf{67.4 \pm 1.0} $	&   $ \mathbf{81.6 \pm 0.4} $	\\
MASS, $\beta$=\num{0}	&   $ \mathbf{50.2 \pm 1.0} $	&   $ \mathbf{67.4 \pm 0.3} $	&   $ \mathbf{81.5 \pm 0.2} $	\\

\bottomrule
\end{tabular}
\end{small}
\end{center}
\vskip -0.1in
\end{table}

\subsection{Uncertainty Quantification}

We also evaluate the ability of networks trained by MASS Learning to properly quantify their uncertainty about their predictions.  We assess uncertainty quantification in two ways: using proper scoring rules \cite{lakshminarayanan_simple_2016}, which are scalar measures of how well a network's predictive distribution $p(y|f_\theta(x))$ is calibrated, and by assessing performance on an out-of-distribution (OOD) detection task.

Tables \ref{tab:SmallMLPUQ-40k} through \ref{tab:ResNetUQ-2.5k} show the uncertainty quantification performance of networks according to two proper scoring rules: the Negative Log Likelihood (NLL) and the Brier Score.  The entropy and test accuracy of the predictive distributions are also given, for reference.

For the SmallMLP network in Tables \ref{tab:SmallMLPUQ-40k}, \ref{tab:SmallMLPUQ-10k}, and \ref{tab:SmallMLPUQ-2.5k}, VIB provides the best combination of high accuracy and low NLL and Brier score across all sizes of training set, despite SoftmaxCE with weight decay achieving the best scoring rule values.  For the larger ResNet20 network in Tables \ref{tab:ResNetUQ-40k} and \ref{tab:ResNetUQ-10k}, MASS Learning provides the best combination of accuracy and proper scoring rule performance, though its performance falters when trained on only 2,500 datapoints in Table and \ref{tab:ResNetUQ-2.5k}.  These ResNet20 UQ results also  show the trend that MASS Learning with larger $\beta$ leads to better calibrated network predictions.   Thus, as measured by proper scoring rules, MASS Learning can significantly improve the calibration of a network's predictions while maintaining the same accuracy.

Tables \ref{tab:SmallMLP-OOD-40k} through \ref{tab:ResNet-OOD-2.5k} show metrics for performance on an OOD detection task where the network predicts not just the class of the input image, but whether the image is from its training distribution (CIFAR-10 images) or from another distribution (SVHN images \cite{Netzer2011ReadingDI}).  Following Hendrycks \& Gimpel \yrcite{hendrycks_baseline_2016} and Alemi et al. \yrcite{alemi_uncertainty_2018}, the metrics we report for this task are the Area under the ROC curve (AUROC) and Average Precision score (APR).  APR depends on whether the network is tasked with identifying in-distribution or out-of-distribution images; we report values for both cases as APR In and APR Out, respectively. 

There are different detection methods that networks can use to identify OOD inputs.  One way, applicable to all training methods, is to use the entropy of the predictive distribution $p(y|f_\theta(x))$: larger entropy suggests the input is OOD.  For networks trained by MASS Learning, the variational distribution $q_\phi(f_\theta(x) | y)$ is a natural OOD detector: a small value of $\max_i q_\phi(f_\theta(x) | y_i)$ suggests the input is OOD.  For networks trained by SoftmaxCE, a distribution $q_\phi(f_\theta(x) | y)$ can be learned by MLE on the training set and used to detect OOD inputs in the same way.

For both the SmallMLP network in Tables \ref{tab:SmallMLP-OOD-40k}, \ref{tab:SmallMLP-OOD-10k}, and \ref{tab:SmallMLP-OOD-2.5k} and the ResNet20 network in Tables \ref{tab:ResNet-OOD-40k}, \ref{tab:ResNet-OOD-10k}, and \ref{tab:ResNet-OOD-2.5k}, MASS Learning performs comparably or better than SoftmaxCE and VIB.  However, one should note that MASS Learning with $\beta=0$ gives performance not significantly different to MASS Learning with $\beta \neq 0$ on these OOD tasks, which suggests that the good performance of MASS Learning may be due to its use of a variational distribution to produce predictions, rather than to the overall MASS Learning training scheme.

\begin{table*}
\caption{Uncertainty quantification metrics (proper scoring rules) on CIFAR-10 using the SmallMLP network trained on 40,000 datapoints.  Test Accuracy and Entropy of the network's predictive distribution are given for reference. Full experiment details are in Supplementary Material \ref{supp:ExperimentDetails}.  Values are the mean over 4 training runs with different random seeds, plus or minus the standard deviation. Emboldened values are those for which the minimum observed mean value in the column was within one standard deviation.   WD is weight decay; D is dropout. Lower values are better.}
\label{tab:SmallMLPUQ-40k}
\vskip 0.15in
\begin{center}
\begin{small}
\begin{tabular}{@{}lll|ll@{}}
\toprule
Method & Test Accuracy & Entropy & NLL & Brier Score  \\
\midrule

SoftmaxCE	&   $ 52.7 \pm 0.4 $	&   $ 0.211 \pm 0.003 $	&   $ 4.56 \pm 0.07 $	&   $ 0.0840 \pm 0.0005 $	\\
SoftmaxCE, WD	&   $ 48.1 \pm 0.1 $	&   $ 1.500 \pm 0.009 $	&   $ \mathbf{1.47 \pm 0.01} $	&   $ \mathbf{0.0660 \pm 0.0003} $	\\
SoftmaxCE, D	&   $ 53.7 \pm 0.3 $	&   $ 0.606 \pm 0.005 $	&   $ 1.79 \pm 0.02 $	&   $ 0.0681 \pm 0.0005 $	\\
VIB, $\beta$=\num{1e-1}	&   $ 46.1 \pm 0.5 $	&   $ 0.258 \pm 0.005 $	&   $ 5.35 \pm 0.15 $	&   $ 0.0944 \pm 0.0009 $	\\
VIB, $\beta$=\num{1e-2}	&   $ 51.9 \pm 0.8 $	&   $ 0.193 \pm 0.004 $	&   $ 5.03 \pm 0.19 $	&   $ 0.0861 \pm 0.0015 $	\\
VIB, $\beta$=\num{1e-3}	&   $ 51.8 \pm 0.8 $	&   $ 0.174 \pm 0.003 $	&   $ 5.49 \pm 0.20 $	&   $ 0.0866 \pm 0.0015 $	\\
VIB, $\beta$=\num{1e-1}, D	&   $ 49.5 \pm 0.5 $	&   $ 0.957 \pm 0.005 $	&   $ 1.62 \pm 0.01 $	&   $ \mathbf{0.0660 \pm 0.0003} $	\\
VIB, $\beta$=\num{1e-2}, D	&   $ 53.6 \pm 0.3 $	&   $ 0.672 \pm 0.014 $	&   $ 1.69 \pm 0.01 $	&   $ 0.0668 \pm 0.0006 $	\\
VIB, $\beta$=\num{1e-3}, D	&   $ 54.3 \pm 0.2 $	&   $ 0.617 \pm 0.007 $	&   $ 1.75 \pm 0.02 $	&   $ 0.0677 \pm 0.0005 $	\\
MASS, $\beta$=\num{1e-2}	&   $ 46.3 \pm 1.2 $	&   $ 0.203 \pm 0.005 $	&   $ 6.89 \pm 0.16 $	&   $ 0.0968 \pm 0.0024 $	\\
MASS, $\beta$=\num{1e-3}	&   $ 47.8 \pm 0.8 $	&   $ 0.207 \pm 0.004 $	&   $ 5.89 \pm 0.21 $	&   $ 0.0935 \pm 0.0017 $	\\
MASS, $\beta$=\num{1e-4}	&   $ 47.9 \pm 0.8 $	&   $ 0.212 \pm 0.003 $	&   $ 5.71 \pm 0.16 $	&   $ 0.0934 \pm 0.0017 $	\\
MASS, $\beta$=\num{0}	&   $ 48.2 \pm 0.9 $	&   $ 0.208 \pm 0.004 $	&   $ 5.74 \pm 0.20 $	&   $ 0.0927 \pm 0.0017 $	\\
MASS, $\beta$=\num{1e-2}, D	&   $ 52.0 \pm 0.6 $	&   $ 0.690 \pm 0.013 $	&   $ 1.85 \pm 0.03 $	&   $ 0.0694 \pm 0.0005 $	\\
MASS, $\beta$=\num{1e-3}, D	&   $ 53.1 \pm 0.4 $	&   $ 0.649 \pm 0.010 $	&   $ 1.82 \pm 0.04 $	&   $ 0.0684 \pm 0.0007 $	\\
MASS, $\beta$=\num{1e-4}, D	&   $ 53.2 \pm 0.1 $	&   $ 0.664 \pm 0.020 $	&   $ 1.79 \pm 0.02 $	&   $ 0.0680 \pm 0.0002 $	\\
MASS, $\beta$=\num{0}, D	&   $ 52.7 \pm 0.0 $	&   $ 0.662 \pm 0.003 $	&   $ 1.82 \pm 0.02 $	&   $ 0.0690 \pm 0.0003 $	\\

\bottomrule
\end{tabular}
\end{small}
\end{center}
\vskip -0.1in
\end{table*}

\begin{table*}
\caption{Uncertainty quantification metrics (proper scoring rules) on CIFAR-10 using the SmallMLP network trained on 10,000 datapoints.  Test Accuracy and Entropy of the network's predictive distribution are given for reference. Full experiment details are in Supplementary Material \ref{supp:ExperimentDetails}.  Values are the mean over 4 training runs with different random seeds, plus or minus the standard deviation. Emboldened values are those for which the minimum observed mean value in the column was within one standard deviation.  WD is weight decay; D is dropout. Lower values are better.}
\label{tab:SmallMLPUQ-10k}
\vskip 0.15in
\begin{center}
\begin{small}
\begin{tabular}{@{}lll|ll@{}}
\toprule
Method & Test Accuracy & Entropy & NLL & Brier Score  \\
\midrule

SoftmaxCE	&   $ 44.6 \pm 0.6 $	&   $ 0.250 \pm 0.004 $	&   $ 5.33 \pm 0.06 $	&   $ 0.0974 \pm 0.0011 $	\\
SoftmaxCE, WD	&   $ 36.4 \pm 0.9 $	&   $ 0.897 \pm 0.033 $	&   $ \mathbf{2.44 \pm 0.11} $	&   $ \mathbf{0.0905 \pm 0.0019} $	\\
SoftmaxCE, D	&   $ 44.1 \pm 0.6 $	&   $ 0.379 \pm 0.007 $	&   $ 3.76 \pm 0.04 $	&   $ 0.0935 \pm 0.0012 $	\\
VIB, $\beta$=\num{1e-1}	&   $ 40.6 \pm 0.4 $	&   $ 0.339 \pm 0.011 $	&   $ 4.86 \pm 0.23 $	&   $ 0.1017 \pm 0.0016 $	\\
VIB, $\beta$=\num{1e-2}	&   $ 43.8 \pm 0.8 $	&   $ 0.274 \pm 0.004 $	&   $ 4.83 \pm 0.16 $	&   $ 0.0983 \pm 0.0017 $	\\
VIB, $\beta$=\num{1e-3}	&   $ 44.6 \pm 0.6 $	&   $ 0.241 \pm 0.004 $	&   $ 5.50 \pm 0.11 $	&   $ 0.0983 \pm 0.0005 $	\\
VIB, $\beta$=\num{1e-1}, D	&   $ 40.1 \pm 0.5 $	&   $ 0.541 \pm 0.015 $	&   $ 3.22 \pm 0.09 $	&   $ 0.0945 \pm 0.0012 $	\\
VIB, $\beta$=\num{1e-2}, D	&   $ 43.9 \pm 0.3 $	&   $ 0.413 \pm 0.009 $	&   $ 3.43 \pm 0.09 $	&   $ 0.0927 \pm 0.0011 $	\\
VIB, $\beta$=\num{1e-3}, D	&   $ 44.4 \pm 0.4 $	&   $ 0.389 \pm 0.004 $	&   $ 3.61 \pm 0.06 $	&   $ 0.0927 \pm 0.0004 $	\\
MASS, $\beta$=\num{1e-2}	&   $ 39.9 \pm 1.2 $	&   $ 0.172 \pm 0.008 $	&   $ 10.06 \pm 0.37 $	&   $ 0.1109 \pm 0.0020 $	\\
MASS, $\beta$=\num{1e-3}	&   $ 41.5 \pm 0.7 $	&   $ 0.197 \pm 0.005 $	&   $ 8.03 \pm 0.28 $	&   $ 0.1069 \pm 0.0016 $	\\
MASS, $\beta$=\num{1e-4}	&   $ 41.5 \pm 1.1 $	&   $ 0.208 \pm 0.008 $	&   $ 7.55 \pm 0.44 $	&   $ 0.1054 \pm 0.0023 $	\\
MASS, $\beta$=\num{0}	&   $ 42.0 \pm 0.6 $	&   $ 0.215 \pm 0.009 $	&   $ 7.21 \pm 0.28 $	&   $ 0.1043 \pm 0.0015 $	\\
MASS, $\beta$=\num{1e-2}, D	&   $ 41.7 \pm 0.4 $	&   $ 0.399 \pm 0.017 $	&   $ 4.21 \pm 0.17 $	&   $ 0.0974 \pm 0.0013 $	\\
MASS, $\beta$=\num{1e-3}, D	&   $ 43.7 \pm 0.2 $	&   $ 0.412 \pm 0.010 $	&   $ 3.71 \pm 0.07 $	&   $ 0.0930 \pm 0.0006 $	\\
MASS, $\beta$=\num{1e-4}, D	&   $ 43.4 \pm 0.5 $	&   $ 0.435 \pm 0.011 $	&   $ 3.50 \pm 0.05 $	&   $ 0.0923 \pm 0.0005 $	\\
MASS, $\beta$=\num{0}, D	&   $ 43.9 \pm 0.4 $	&   $ 0.447 \pm 0.009 $	&   $ 3.40 \pm 0.03 $	&   $ \mathbf{0.0913 \pm 0.0008} $	\\

\bottomrule
\end{tabular}
\end{small}
\end{center}
\vskip -0.1in
\end{table*}

\begin{table*}
\caption{Uncertainty quantification metrics (proper scoring rules) on CIFAR-10 using the SmallMLP network trained on 2,500 datapoints. Test Accuracy and Entropy of the network's predictive distribution are given for reference.  Full experiment details are in Supplementary Material \ref{supp:ExperimentDetails}.  Values are the mean over 4 training runs with different random seeds, plus or minus the standard deviation. Emboldened values are those for which the minimum observed mean value in the column was within one standard deviation.  WD is weight decay; D is dropout. Lower values are better.}
\label{tab:SmallMLPUQ-2.5k}
\vskip 0.15in
\begin{center}
\begin{small}
\begin{tabular}{@{}lll|ll@{}}
\toprule
Method & Test Accuracy & Entropy & NLL & Brier Score  \\
\midrule

SoftmaxCE	&   $ 34.2 \pm 0.8 $	&   $ 0.236 \pm 0.025 $	&   $ 8.14 \pm 0.84 $	&   $ 0.1199 \pm 0.0024 $	\\
SoftmaxCE, WD	&   $ 23.9 \pm 0.9 $	&   $ 0.954 \pm 0.017 $	&   $ \mathbf{3.41 \pm 0.07} $	&   $ \mathbf{0.1114 \pm 0.0013} $	\\
SoftmaxCE, D	&   $ 33.7 \pm 1.1 $	&   $ 0.203 \pm 0.006 $	&   $ 9.68 \pm 0.06 $	&   $ 0.1219 \pm 0.0013 $	\\
VIB, $\beta$=\num{1e-1}	&   $ 32.2 \pm 0.6 $	&   $ 0.247 \pm 0.007 $	&   $ 8.33 \pm 0.50 $	&   $ 0.1219 \pm 0.0013 $	\\
VIB, $\beta$=\num{1e-2}	&   $ 34.6 \pm 0.4 $	&   $ 0.249 \pm 0.004 $	&   $ 7.36 \pm 0.18 $	&   $ 0.1175 \pm 0.0005 $	\\
VIB, $\beta$=\num{1e-3}	&   $ 35.6 \pm 0.5 $	&   $ 0.217 \pm 0.008 $	&   $ 8.03 \pm 0.37 $	&   $ 0.1175 \pm 0.0012 $	\\
VIB, $\beta$=\num{1e-1}, D	&   $ 29.0 \pm 0.6 $	&   $ 0.383 \pm 0.011 $	&   $ 6.32 \pm 0.16 $	&   $ 0.1219 \pm 0.0010 $	\\
VIB, $\beta$=\num{1e-2}, D	&   $ 32.5 \pm 0.9 $	&   $ 0.260 \pm 0.006 $	&   $ 7.41 \pm 0.25 $	&   $ 0.1211 \pm 0.0019 $	\\
VIB, $\beta$=\num{1e-3}, D	&   $ 34.5 \pm 1.0 $	&   $ 0.200 \pm 0.002 $	&   $ 9.44 \pm 0.16 $	&   $ 0.1203 \pm 0.0020 $	\\
MASS, $\beta$=\num{1e-2}	&   $ 29.6 \pm 0.4 $	&   $ 0.047 \pm 0.002 $	&   $ 57.13 \pm 1.60 $	&   $ 0.1381 \pm 0.0007 $	\\
MASS, $\beta$=\num{1e-3}	&   $ 32.7 \pm 0.8 $	&   $ 0.048 \pm 0.004 $	&   $ 46.40 \pm 3.81 $	&   $ 0.1322 \pm 0.0018 $	\\
MASS, $\beta$=\num{1e-4}	&   $ 34.0 \pm 0.3 $	&   $ 0.052 \pm 0.002 $	&   $ 39.10 \pm 1.96 $	&   $ 0.1293 \pm 0.0009 $	\\
MASS, $\beta$=\num{0}	&   $ 34.1 \pm 0.6 $	&   $ 0.061 \pm 0.003 $	&   $ 33.60 \pm 1.34 $	&   $ 0.1285 \pm 0.0012 $	\\
MASS, $\beta$=\num{1e-2}, D	&   $ 29.3 \pm 1.2 $	&   $ 0.118 \pm 0.008 $	&   $ 20.51 \pm 0.83 $	&   $ 0.1349 \pm 0.0018 $	\\
MASS, $\beta$=\num{1e-3}, D	&   $ 31.5 \pm 0.6 $	&   $ 0.145 \pm 0.004 $	&   $ 15.65 \pm 0.71 $	&   $ 0.1289 \pm 0.0010 $	\\
MASS, $\beta$=\num{1e-4}, D	&   $ 32.7 \pm 0.8 $	&   $ 0.185 \pm 0.010 $	&   $ 11.21 \pm 0.66 $	&   $ 0.1245 \pm 0.0011 $	\\
MASS, $\beta$=\num{0}, D	&   $ 32.2 \pm 1.1 $	&   $ 0.217 \pm 0.008 $	&   $ 9.70 \pm 0.29 $	&   $ 0.1236 \pm 0.0021 $	\\

\bottomrule
\end{tabular}
\end{small}
\end{center}
\vskip -0.1in
\end{table*}

\begin{table*}
\caption{Uncertainty quantification metrics (proper scoring rules) on CIFAR-10 using the ResNet20 network trained on 40,000 datapoints. Test Accuracy and Entropy of the network's predictive distribution are given for reference. Full experiment details are in Supplementary Material \ref{supp:ExperimentDetails}.  Values are the mean over 4 training runs with different random seeds, plus or minus the standard deviation. Emboldened values are those for which the minimum observed mean value in the column was within one standard deviation. Lower values are better.}
\label{tab:ResNetUQ-40k}
\vskip 0.15in
\begin{center}
\begin{small}
\begin{tabular}{@{}lll|ll@{}}
\toprule
Method & Test Accuracy & Entropy & NLL & Brier Score  \\
\midrule

SoftmaxCE	&   $ 81.7 \pm 0.3 $	&   $ 0.087 \pm 0.002 $	&   $ 1.45 \pm 0.04 $	&   $ \mathbf{0.0324 \pm 0.0005} $	\\
VIB, $\beta$=\num{1e-3}	&   $ 81.0 \pm 0.3 $	&   $ 0.089 \pm 0.003 $	&   $ 1.51 \pm 0.04 $	&   $ 0.0334 \pm 0.0005 $	\\
VIB, $\beta$=\num{1e-4}	&   $ 81.2 \pm 0.4 $	&   $ 0.092 \pm 0.002 $	&   $ 1.46 \pm 0.05 $	&   $ 0.0331 \pm 0.0007 $	\\
VIB, $\beta$=\num{1e-5}	&   $ 80.9 \pm 0.5 $	&   $ 0.087 \pm 0.005 $	&   $ 1.58 \pm 0.08 $	&   $ 0.0339 \pm 0.0008 $	\\
VIB, $\beta$=\num{0}	&   $ 81.5 \pm 0.2 $	&   $ 0.079 \pm 0.001 $	&   $ 1.70 \pm 0.06 $	&   $ 0.0331 \pm 0.0007 $	\\
MASS, $\beta$=\num{1e-3}	&   $ 75.8 \pm 0.5 $	&   $ 0.139 \pm 0.003 $	&   $ 1.66 \pm 0.07 $	&   $ 0.0417 \pm 0.0011 $	\\
MASS, $\beta$=\num{1e-4}	&   $ 80.6 \pm 0.5 $	&   $ 0.109 \pm 0.002 $	&   $ \mathbf{1.33 \pm 0.02} $	&   $ 0.0337 \pm 0.0008 $	\\
MASS, $\beta$=\num{1e-5}	&   $ 81.6 \pm 0.4 $	&   $ 0.095 \pm 0.003 $	&   $ \mathbf{1.36 \pm 0.03} $	&   $ \mathbf{0.0320 \pm 0.0005} $	\\
MASS, $\beta$=\num{0}	&   $ 81.5 \pm 0.2 $	&   $ 0.092 \pm 0.000 $	&   $ 1.43 \pm 0.04 $	&   $ 0.0325 \pm 0.0004 $	\\

\bottomrule
\end{tabular}
\end{small}
\end{center}
\vskip -0.1in
\end{table*}

\begin{table*}
\caption{Uncertainty quantification metrics (proper scoring rules) on CIFAR-10 using the ResNet20 network trained on 10,000 datapoints. Test Accuracy and Entropy of the network's predictive distribution are given for reference. Full experiment details are in Supplementary Material \ref{supp:ExperimentDetails}.  Values are the mean over 4 training runs with different random seeds, plus or minus the standard deviation. Emboldened values are those for which the minimum observed mean value in the column was within one standard deviation. Lower values are better.}
\label{tab:ResNetUQ-10k}
\vskip 0.15in
\begin{center}
\begin{small}
\begin{tabular}{@{}lll|ll@{}}
\toprule
Method & Test Accuracy & Entropy & NLL & Brier Score \\
\midrule

SoftmaxCE	&   $ 67.5 \pm 0.8 $	&   $ 0.195 \pm 0.011 $	&   $ \mathbf{2.19 \pm 0.06} $	&   $ \mathbf{0.0557 \pm 0.0012} $	\\
VIB, $\beta$=\num{1e-3}	&   $ 66.9 \pm 1.0 $	&   $ 0.193 \pm 0.008 $	&   $ \mathbf{2.26 \pm 0.13} $	&   $ 0.0570 \pm 0.0017 $	\\
VIB, $\beta$=\num{1e-4}	&   $ 66.4 \pm 0.5 $	&   $ 0.197 \pm 0.009 $	&   $ 2.30 \pm 0.02 $	&   $ 0.0577 \pm 0.0007 $	\\
VIB, $\beta$=\num{1e-5}	&   $ 67.9 \pm 0.8 $	&   $ 0.166 \pm 0.010 $	&   $ 2.49 \pm 0.13 $	&   $ \mathbf{0.0561 \pm 0.0011} $	\\
VIB, $\beta$=\num{0}	&   $ 67.1 \pm 1.0 $	&   $ 0.162 \pm 0.009 $	&   $ 2.64 \pm 0.11 $	&   $ 0.0578 \pm 0.0016 $	\\
MASS, $\beta$=\num{1e-3}	&   $ 59.6 \pm 0.8 $	&   $ 0.252 \pm 0.007 $	&   $ 2.61 \pm 0.11 $	&   $ 0.0688 \pm 0.0014 $	\\
MASS, $\beta$=\num{1e-4}	&   $ 66.6 \pm 0.4 $	&   $ 0.209 \pm 0.009 $	&   $ \mathbf{2.18 \pm 0.05} $	&   $ 0.0570 \pm 0.0005 $	\\
MASS, $\beta$=\num{1e-5}	&   $ 67.4 \pm 1.0 $	&   $ 0.192 \pm 0.007 $	&   $ \mathbf{2.22 \pm 0.07} $	&   $ \mathbf{0.0561 \pm 0.0017} $	\\
MASS, $\beta$=\num{0}	&   $ 67.4 \pm 0.3 $	&   $ 0.189 \pm 0.004 $	&   $ 2.30 \pm 0.08 $	&   $ \mathbf{0.0562 \pm 0.0007} $	\\

\bottomrule
\end{tabular}
\end{small}
\end{center}
\vskip -0.1in
\end{table*}

\begin{table*}
\caption{Uncertainty quantification metrics (proper scoring rules) on CIFAR-10 using the ResNet20 network trained on 2,500 datapoints. Test Accuracy and Entropy of the network's predictive distribution are given for reference. Full experiment details are in Supplementary Material \ref{supp:ExperimentDetails}.  Values are the mean over 4 training runs with different random seeds, plus or minus the standard deviation. Emboldened values are those for which the minimum observed mean value in the column was within one standard deviation. Lower values are better.}
\label{tab:ResNetUQ-2.5k}
\vskip 0.15in
\begin{center}
\begin{small}
\begin{tabular}{@{}lll|ll@{}}
\toprule
Method & Test Accuracy & Entropy & NLL & Brier Score  \\
\midrule

SoftmaxCE	&   $ 50.0 \pm 0.7 $	&   $ 0.349 \pm 0.005 $	&   $ \mathbf{2.98 \pm 0.06} $	&   $ \mathbf{0.0833 \pm 0.0012} $	\\
VIB, $\beta$=\num{1e-3}	&   $ 49.5 \pm 1.1 $	&   $ 0.363 \pm 0.005 $	&   $ 3.10 \pm 0.11 $	&   $ \mathbf{0.0836 \pm 0.0020} $	\\
VIB, $\beta$=\num{1e-4}	&   $ 49.4 \pm 1.0 $	&   $ 0.372 \pm 0.016 $	&   $ \mathbf{3.02 \pm 0.10} $	&   $ \mathbf{0.0833 \pm 0.0016} $	\\
VIB, $\beta$=\num{1e-5}	&   $ 50.0 \pm 1.1 $	&   $ 0.306 \pm 0.021 $	&   $ 3.48 \pm 0.15 $	&   $ 0.0849 \pm 0.0013 $	\\
VIB, $\beta$=\num{0}	&   $ 50.6 \pm 0.8 $	&   $ 0.271 \pm 0.019 $	&   $ 3.80 \pm 0.15 $	&   $ 0.0850 \pm 0.0007 $	\\
MASS, $\beta$=\num{1e-3}	&   $ 38.2 \pm 0.7 $	&   $ 0.469 \pm 0.012 $	&   $ 3.75 \pm 0.08 $	&   $ 0.1010 \pm 0.0017 $	\\
MASS, $\beta$=\num{1e-4}	&   $ 49.9 \pm 1.0 $	&   $ 0.344 \pm 0.001 $	&   $ 3.24 \pm 0.08 $	&   $ \mathbf{0.0837 \pm 0.0017} $	\\
MASS, $\beta$=\num{1e-5}	&   $ 50.1 \pm 0.5 $	&   $ 0.277 \pm 0.008 $	&   $ 3.81 \pm 0.11 $	&   $ 0.0859 \pm 0.0005 $	\\
MASS, $\beta$=\num{0}	&   $ 50.2 \pm 1.0 $	&   $ 0.265 \pm 0.009 $	&   $ 3.96 \pm 0.15 $	&   $ 0.0861 \pm 0.0020 $	\\

\bottomrule
\end{tabular}
\end{small}
\end{center}
\vskip -0.1in
\end{table*}

\begin{table*}[h]
\caption{Out-of-distribution detection metrics for SmallMLP network trained on 40,000 CIFAR-10 images, with SVHN as the out-of-distribution examples.  Full experiment details are in Supplementary Material \ref{supp:ExperimentDetails}.  Values are the mean over 4 training runs with different random seeds, plus or minus the standard deviation. Emboldened values are those for which the maximum observed mean value in the column was within one standard deviation.   WD is weight decay; D is dropout. Higher values are better.}
\label{tab:SmallMLP-OOD-40k}
\vskip 0.15in
\begin{center}
\begin{small}
\begin{tabular}{@{}lc|cccc@{}}
\toprule
Training Method & Test Accuracy & Detection Method & AUROC & APR In & APR Out \\
\midrule

\multirow{2}{*}{	SoftmaxCE	}	& \multirow{2}{*}{$ 52.7 \pm  0.4$}	& Entropy	&   $ 0.65 \pm 0.01 $	&   $ 0.68 \pm 0.01 $	&   $ 0.61 \pm 0.01 $	\\	& & $\max_i q_\phi(f_\theta(x) | y_i)$	&   $ 0.38 \pm 0.01 $	&   $ 0.42 \pm 0.01 $	&   $ 0.43 \pm 0.01 $	\\ \cmidrule(r){1-2}
\multirow{2}{*}{	SoftmaxCE, WD	}	& \multirow{2}{*}{$ 48.1 \pm  0.1$}	& Entropy	&   $ 0.65 \pm 0.01 $	&   $ 0.69 \pm 0.01 $	&   $ 0.59 \pm 0.01 $	\\	& & $\max_i q_\phi(f_\theta(x) | y_i)$	&   $ 0.43 \pm 0.01 $	&   $ 0.43 \pm 0.01 $	&   $ 0.48 \pm 0.02 $	\\ \cmidrule(r){1-2}
\multirow{2}{*}{	SoftmaxCE, D	}	& \multirow{2}{*}{$ 53.7 \pm  0.3$}	& Entropy	&   $ 0.71 \pm 0.01 $	&   $ \mathbf{0.75 \pm 0.01} $	&   $ 0.65 \pm 0.01 $	\\	& & $\max_i q_\phi(f_\theta(x) | y_i)$	&   $ 0.33 \pm 0.00 $	&   $ 0.39 \pm 0.00 $	&   $ 0.40 \pm 0.00 $	\\ \cmidrule(r){1-2}
\multirow{2}{*}{	VIB, $\beta$=\num{1e-1}	}	& \multirow{2}{*}{$ 46.1 \pm  0.5$}	& Entropy	&   $ 0.62 \pm 0.01 $	&   $ 0.66 \pm 0.01 $	&   $ 0.57 \pm 0.01 $	\\	& & Rate	&   $ 0.47 \pm 0.02 $	&   $ 0.49 \pm 0.01 $	&   $ 0.46 \pm 0.01 $	\\ \cmidrule(r){1-2}
\multirow{2}{*}{	VIB, $\beta$=\num{1e-2}	}	& \multirow{2}{*}{$ 51.9 \pm  0.8$}	& Entropy	&   $ 0.64 \pm 0.01 $	&   $ 0.67 \pm 0.01 $	&   $ 0.59 \pm 0.01 $	\\	& & Rate	&   $ 0.58 \pm 0.03 $	&   $ 0.59 \pm 0.02 $	&   $ 0.55 \pm 0.02 $	\\ \cmidrule(r){1-2}
\multirow{2}{*}{	VIB, $\beta$=\num{1e-3}	}	& \multirow{2}{*}{$ 51.8 \pm  0.8$}	& Entropy	&   $ 0.65 \pm 0.00 $	&   $ 0.67 \pm 0.01 $	&   $ 0.61 \pm 0.00 $	\\	& & Rate	&   $ 0.52 \pm 0.03 $	&   $ 0.54 \pm 0.03 $	&   $ 0.50 \pm 0.03 $	\\ \cmidrule(r){1-2}
\multirow{2}{*}{	VIB, $\beta$=\num{1e-1}, D	}	& \multirow{2}{*}{$ 49.5 \pm  0.5$}	& Entropy	&   $ 0.68 \pm 0.01 $	&   $ \mathbf{0.74 \pm 0.01} $	&   $ 0.60 \pm 0.01 $	\\	& & Rate	&   $ 0.34 \pm 0.01 $	&   $ 0.40 \pm 0.01 $	&   $ 0.39 \pm 0.00 $	\\ \cmidrule(r){1-2}
\multirow{2}{*}{	VIB, $\beta$=\num{1e-2}, D	}	& \multirow{2}{*}{$ 53.6 \pm  0.3$}	& Entropy	&   $ 0.69 \pm 0.02 $	&   $ 0.73 \pm 0.01 $	&   $ 0.62 \pm 0.02 $	\\	& & Rate	&   $ 0.50 \pm 0.03 $	&   $ 0.51 \pm 0.02 $	&   $ 0.51 \pm 0.03 $	\\ \cmidrule(r){1-2}
\multirow{2}{*}{	VIB, $\beta$=\num{1e-3}, D	}	& \multirow{2}{*}{$ 54.3 \pm  0.2$}	& Entropy	&   $ 0.69 \pm 0.01 $	&   $ 0.73 \pm 0.01 $	&   $ 0.62 \pm 0.01 $	\\	& & Rate	&   $ 0.45 \pm 0.01 $	&   $ 0.45 \pm 0.01 $	&   $ 0.49 \pm 0.01 $	\\ \cmidrule(r){1-2}
\multirow{2}{*}{	MASS, $\beta$=\num{1e-2}	}	& \multirow{2}{*}{$ 46.3 \pm  1.2$}	& Entropy	&   $ 0.64 \pm 0.01 $	&   $ 0.67 \pm 0.01 $	&   $ 0.61 \pm 0.01 $	\\	& & $\max_i q_\phi(f_\theta(x) | y_i)$	&   $ 0.51 \pm 0.03 $	&   $ 0.56 \pm 0.05 $	&   $ 0.49 \pm 0.01 $	\\ \cmidrule(r){1-2}
\multirow{2}{*}{	MASS, $\beta$=\num{1e-3}	}	& \multirow{2}{*}{$ 47.8 \pm  0.8$}	& Entropy	&   $ 0.63 \pm 0.02 $	&   $ 0.65 \pm 0.02 $	&   $ 0.60 \pm 0.02 $	\\	& & $\max_i q_\phi(f_\theta(x) | y_i)$	&   $ 0.63 \pm 0.07 $	&   $ 0.64 \pm 0.08 $	&   $ 0.60 \pm 0.05 $	\\ \cmidrule(r){1-2}
\multirow{2}{*}{	MASS, $\beta$=\num{1e-4}	}	& \multirow{2}{*}{$ 47.9 \pm  0.8$}	& Entropy	&   $ 0.63 \pm 0.02 $	&   $ 0.65 \pm 0.02 $	&   $ 0.60 \pm 0.02 $	\\	& & $\max_i q_\phi(f_\theta(x) | y_i)$	&   $ 0.57 \pm 0.06 $	&   $ 0.58 \pm 0.05 $	&   $ 0.56 \pm 0.05 $	\\ \cmidrule(r){1-2}
\multirow{2}{*}{	MASS, $\beta$=\num{0}	}	& \multirow{2}{*}{$ 48.2 \pm  0.9$}	& Entropy	&   $ 0.63 \pm 0.02 $	&   $ 0.65 \pm 0.02 $	&   $ 0.59 \pm 0.02 $	\\	& & $\max_i q_\phi(f_\theta(x) | y_i)$	&   $ 0.58 \pm 0.06 $	&   $ 0.58 \pm 0.05 $	&   $ 0.56 \pm 0.05 $	\\ \cmidrule(r){1-2}
\multirow{2}{*}{	MASS, $\beta$=\num{1e-2}, D	}	& \multirow{2}{*}{$ 52.0 \pm  0.6$}	& Entropy	&   $ \mathbf{0.73 \pm 0.01} $	&   $ \mathbf{0.75 \pm 0.01} $	&   $ \mathbf{0.67 \pm 0.01} $	\\	& & $\max_i q_\phi(f_\theta(x) | y_i)$	&   $ 0.65 \pm 0.06 $	&   $ 0.70 \pm 0.06 $	&   $ 0.58 \pm 0.05 $	\\ \cmidrule(r){1-2}
\multirow{2}{*}{	MASS, $\beta$=\num{1e-3}, D	}	& \multirow{2}{*}{$ 53.1 \pm  0.4$}	& Entropy	&   $ \mathbf{0.71 \pm 0.02} $	&   $ 0.73 \pm 0.01 $	&   $ 0.64 \pm 0.02 $	\\	& & $\max_i q_\phi(f_\theta(x) | y_i)$	&   $ 0.64 \pm 0.10 $	&   $ 0.66 \pm 0.10 $	&   $ 0.60 \pm 0.09 $	\\ \cmidrule(r){1-2}
\multirow{2}{*}{	MASS, $\beta$=\num{1e-4}, D	}	& \multirow{2}{*}{$ 53.2 \pm  0.1$}	& Entropy	&   $ \mathbf{0.73 \pm 0.01} $	&   $ \mathbf{0.75 \pm 0.01} $	&   $ \mathbf{0.67 \pm 0.01} $	\\	& & $\max_i q_\phi(f_\theta(x) | y_i)$	&   $ 0.65 \pm 0.09 $	&   $ 0.65 \pm 0.08 $	&   $ 0.61 \pm 0.08 $	\\ \cmidrule(r){1-2}
\multirow{2}{*}{	MASS, $\beta$=\num{0}, D	}	& \multirow{2}{*}{$ 52.7 \pm  0.0$}	& Entropy	&   $ \mathbf{0.71 \pm 0.02} $	&   $ \mathbf{0.74 \pm 0.01} $	&   $ \mathbf{0.65 \pm 0.02} $	\\	& & $\max_i q_\phi(f_\theta(x) | y_i)$	&   $ 0.63 \pm 0.09 $	&   $ 0.65 \pm 0.08 $	&   $ 0.59 \pm 0.09 $	\\

\bottomrule
\end{tabular}
\end{small}
\end{center}
\vskip -0.1in
\end{table*}

\begin{table*}[h]
\caption{Out-of-distribution detection metrics for SmallMLP network trained on 10,000 CIFAR-10 images, with SVHN as the out-of-distribution examples.  Full experiment details are in Supplementary Material \ref{supp:ExperimentDetails}.  Values are the mean over 4 training runs with different random seeds, plus or minus the standard deviation. Emboldened values are those for which the maximum observed mean value in the column was within one standard deviation.   WD is weight decay; D is dropout. Higher values are better.}
\label{tab:SmallMLP-OOD-10k}
\vskip 0.15in
\begin{center}
\begin{small}
\begin{tabular}{@{}lc|cccc@{}}
\toprule
Training Method & Test Accuracy & Detection Method & AUROC & APR In & APR Out \\
\midrule

\multirow{2}{*}{	SoftmaxCE	}	& \multirow{2}{*}{$ 44.6 \pm  0.6$}	& Entropy	&   $ 0.62 \pm 0.00 $	&   $ 0.64 \pm 0.01 $	&   $ 0.59 \pm 0.00 $	\\	& & $\max_i q_\phi(f_\theta(x) | y_i)$	&   $ 0.36 \pm 0.01 $	&   $ 0.40 \pm 0.01 $	&   $ 0.42 \pm 0.00 $	\\ \cmidrule(r){1-2}
\multirow{2}{*}{	SoftmaxCE, WD	}	& \multirow{2}{*}{$ 36.4 \pm  0.9$}	& Entropy	&   $ 0.62 \pm 0.02 $	&   $ 0.62 \pm 0.02 $	&   $ 0.60 \pm 0.02 $	\\	& & $\max_i q_\phi(f_\theta(x) | y_i)$	&   $ 0.30 \pm 0.01 $	&   $ 0.37 \pm 0.00 $	&   $ 0.39 \pm 0.01 $	\\ \cmidrule(r){1-2}
\multirow{2}{*}{	SoftmaxCE, D	}	& \multirow{2}{*}{$ 44.1 \pm  0.6$}	& Entropy	&   $ 0.66 \pm 0.01 $	&   $ \mathbf{0.69 \pm 0.01} $	&   $ 0.62 \pm 0.01 $	\\	& & $\max_i q_\phi(f_\theta(x) | y_i)$	&   $ 0.29 \pm 0.01 $	&   $ 0.37 \pm 0.00 $	&   $ 0.38 \pm 0.00 $	\\ \cmidrule(r){1-2}
\multirow{2}{*}{	VIB, $\beta$=\num{1e-1}	}	& \multirow{2}{*}{$ 40.6 \pm  0.4$}	& Entropy	&   $ 0.60 \pm 0.01 $	&   $ 0.64 \pm 0.01 $	&   $ 0.56 \pm 0.01 $	\\	& & Rate	&   $ 0.50 \pm 0.02 $	&   $ 0.52 \pm 0.02 $	&   $ 0.48 \pm 0.01 $	\\ \cmidrule(r){1-2}
\multirow{2}{*}{	VIB, $\beta$=\num{1e-2}	}	& \multirow{2}{*}{$ 43.8 \pm  0.8$}	& Entropy	&   $ 0.62 \pm 0.00 $	&   $ 0.64 \pm 0.01 $	&   $ 0.59 \pm 0.01 $	\\	& & Rate	&   $ 0.55 \pm 0.03 $	&   $ 0.57 \pm 0.02 $	&   $ 0.53 \pm 0.02 $	\\ \cmidrule(r){1-2}
\multirow{2}{*}{	VIB, $\beta$=\num{1e-3}	}	& \multirow{2}{*}{$ 44.6 \pm  0.6$}	& Entropy	&   $ 0.62 \pm 0.01 $	&   $ 0.64 \pm 0.01 $	&   $ 0.59 \pm 0.01 $	\\	& & Rate	&   $ 0.49 \pm 0.04 $	&   $ 0.52 \pm 0.04 $	&   $ 0.48 \pm 0.03 $	\\ \cmidrule(r){1-2}
\multirow{2}{*}{	VIB, $\beta$=\num{1e-1}, D	}	& \multirow{2}{*}{$ 40.1 \pm  0.5$}	& Entropy	&   $ 0.62 \pm 0.00 $	&   $ 0.65 \pm 0.01 $	&   $ 0.57 \pm 0.00 $	\\	& & Rate	&   $ 0.49 \pm 0.02 $	&   $ 0.51 \pm 0.02 $	&   $ 0.48 \pm 0.01 $	\\ \cmidrule(r){1-2}
\multirow{2}{*}{	VIB, $\beta$=\num{1e-2}, D	}	& \multirow{2}{*}{$ 43.9 \pm  0.3$}	& Entropy	&   $ \mathbf{0.67 \pm 0.01} $	&   $ \mathbf{0.69 \pm 0.01} $	&   $ 0.62 \pm 0.00 $	\\	& & Rate	&   $ 0.60 \pm 0.02 $	&   $ 0.61 \pm 0.02 $	&   $ 0.56 \pm 0.01 $	\\ \cmidrule(r){1-2}
\multirow{2}{*}{	VIB, $\beta$=\num{1e-3}, D	}	& \multirow{2}{*}{$ 44.4 \pm  0.4$}	& Entropy	&   $ \mathbf{0.67 \pm 0.01} $	&   $ \mathbf{0.69 \pm 0.01} $	&   $ 0.63 \pm 0.01 $	\\	& & Rate	&   $ 0.50 \pm 0.03 $	&   $ 0.53 \pm 0.03 $	&   $ 0.49 \pm 0.02 $	\\ \cmidrule(r){1-2}
\multirow{2}{*}{	MASS, $\beta$=\num{1e-2}	}	& \multirow{2}{*}{$ 39.9 \pm  1.2$}	& Entropy	&   $ 0.63 \pm 0.02 $	&   $ 0.64 \pm 0.02 $	&   $ 0.60 \pm 0.01 $	\\	& & $\max_i q_\phi(f_\theta(x) | y_i)$	&   $ 0.54 \pm 0.03 $	&   $ 0.58 \pm 0.04 $	&   $ 0.50 \pm 0.02 $	\\ \cmidrule(r){1-2}
\multirow{2}{*}{	MASS, $\beta$=\num{1e-3}	}	& \multirow{2}{*}{$ 41.5 \pm  0.7$}	& Entropy	&   $ 0.61 \pm 0.02 $	&   $ 0.62 \pm 0.02 $	&   $ 0.59 \pm 0.01 $	\\	& & $\max_i q_\phi(f_\theta(x) | y_i)$	&   $ 0.59 \pm 0.07 $	&   $ 0.60 \pm 0.06 $	&   $ 0.56 \pm 0.06 $	\\ \cmidrule(r){1-2}
\multirow{2}{*}{	MASS, $\beta$=\num{1e-4}	}	& \multirow{2}{*}{$ 41.5 \pm  1.1$}	& Entropy	&   $ 0.60 \pm 0.00 $	&   $ 0.61 \pm 0.01 $	&   $ 0.58 \pm 0.00 $	\\	& & $\max_i q_\phi(f_\theta(x) | y_i)$	&   $ 0.55 \pm 0.05 $	&   $ 0.56 \pm 0.04 $	&   $ 0.53 \pm 0.04 $	\\ \cmidrule(r){1-2}
\multirow{2}{*}{	MASS, $\beta$=\num{0}	}	& \multirow{2}{*}{$ 42.0 \pm  0.6$}	& Entropy	&   $ 0.60 \pm 0.02 $	&   $ 0.61 \pm 0.02 $	&   $ 0.57 \pm 0.01 $	\\	& & $\max_i q_\phi(f_\theta(x) | y_i)$	&   $ 0.55 \pm 0.06 $	&   $ 0.57 \pm 0.04 $	&   $ 0.54 \pm 0.05 $	\\ \cmidrule(r){1-2}
\multirow{2}{*}{	MASS, $\beta$=\num{1e-2}, D	}	& \multirow{2}{*}{$ 41.7 \pm  0.4$}	& Entropy	&   $ \mathbf{0.67 \pm 0.01} $	&   $ \mathbf{0.68 \pm 0.01} $	&   $ \mathbf{0.63 \pm 0.01} $	\\	& & $\max_i q_\phi(f_\theta(x) | y_i)$	&   $ 0.63 \pm 0.04 $	&   $ \mathbf{0.65 \pm 0.04} $	&   $ 0.57 \pm 0.04 $	\\ \cmidrule(r){1-2}
\multirow{2}{*}{	MASS, $\beta$=\num{1e-3}, D	}	& \multirow{2}{*}{$ 43.7 \pm  0.2$}	& Entropy	&   $ \mathbf{0.67 \pm 0.01} $	&   $ \mathbf{0.68 \pm 0.01} $	&   $ \mathbf{0.63 \pm 0.01} $	\\	& & $\max_i q_\phi(f_\theta(x) | y_i)$	&   $ \mathbf{0.66 \pm 0.05} $	&   $ 0.66 \pm 0.04 $	&   $ \mathbf{0.61 \pm 0.06} $	\\ \cmidrule(r){1-2}
\multirow{2}{*}{	MASS, $\beta$=\num{1e-4}, D	}	& \multirow{2}{*}{$ 43.4 \pm  0.5$}	& Entropy	&   $ \mathbf{0.68 \pm 0.01} $	&   $ \mathbf{0.69 \pm 0.01} $	&   $ \mathbf{0.64 \pm 0.02} $	\\	& & $\max_i q_\phi(f_\theta(x) | y_i)$	&   $ \mathbf{0.64 \pm 0.07} $	&   $ \mathbf{0.65 \pm 0.05} $	&   $ \mathbf{0.59 \pm 0.08} $	\\ \cmidrule(r){1-2}
\multirow{2}{*}{	MASS, $\beta$=\num{0}, D	}	& \multirow{2}{*}{$ 43.9 \pm  0.4$}	& Entropy	&   $ \mathbf{0.68 \pm 0.00} $	&   $ \mathbf{0.69 \pm 0.01} $	&   $ \mathbf{0.64 \pm 0.00} $	\\	& & $\max_i q_\phi(f_\theta(x) | y_i)$	&   $ \mathbf{0.65 \pm 0.04} $	&   $ \mathbf{0.66 \pm 0.03} $	&   $ \mathbf{0.60 \pm 0.06} $	\\ 

\bottomrule
\end{tabular}
\end{small}
\end{center}
\vskip -0.1in
\end{table*}

\begin{table*}[h]
\caption{Out-of-distribution detection metrics for SmallMLP network trained on 2,500 CIFAR-10 images, with SVHN as the out-of-distribution examples.  Full experiment details are in Supplementary Material \ref{supp:ExperimentDetails}.   Values are the mean over 4 training runs with different random seeds, plus or minus the standard deviation. Emboldened values are those for which the maximum observed mean value in the column was within one standard deviation.  WD is weight decay; D is dropout. Higher values are better.}
\label{tab:SmallMLP-OOD-2.5k}
\vskip 0.15in
\begin{center}
\begin{small}
\begin{tabular}{@{}lc|cccc@{}}
\toprule
Training Method & Test Accuracy & Detection Method & AUROC & APR In & APR Out \\
\midrule

\multirow{2}{*}{	SoftmaxCE	}	& \multirow{2}{*}{$ 34.2 \pm  0.8$}	& Entropy	&   $ 0.61 \pm 0.01 $	&   $ 0.62 \pm 0.01 $	&   $ 0.59 \pm 0.01 $	\\	& & $\max_i q_\phi(f_\theta(x) | y_i)$	&   $ 0.30 \pm 0.02 $	&   $ 0.38 \pm 0.01 $	&   $ 0.39 \pm 0.01 $	\\ \cmidrule(r){1-2}
\multirow{2}{*}{	SoftmaxCE, WD	}	& \multirow{2}{*}{$ 23.9 \pm  0.9$}	& Entropy	&   $ \mathbf{0.70 \pm 0.03} $	&   $ \mathbf{0.67 \pm 0.03} $	&   $ \mathbf{0.71 \pm 0.04} $	\\	& & $\max_i q_\phi(f_\theta(x) | y_i)$	&   $ 0.23 \pm 0.02 $	&   $ 0.36 \pm 0.01 $	&   $ 0.36 \pm 0.01 $	\\ \cmidrule(r){1-2}
\multirow{2}{*}{	SoftmaxCE, D	}	& \multirow{2}{*}{$ 33.7 \pm  1.1$}	& Entropy	&   $ 0.60 \pm 0.01 $	&   $ 0.62 \pm 0.01 $	&   $ 0.58 \pm 0.01 $	\\	& & $\max_i q_\phi(f_\theta(x) | y_i)$	&   $ 0.27 \pm 0.01 $	&   $ 0.37 \pm 0.00 $	&   $ 0.37 \pm 0.00 $	\\ \cmidrule(r){1-2}
\multirow{2}{*}{	VIB, $\beta$=\num{1e-1}	}	& \multirow{2}{*}{$ 32.2 \pm  0.6$}	& Entropy	&   $ 0.58 \pm 0.01 $	&   $ 0.60 \pm 0.02 $	&   $ 0.56 \pm 0.01 $	\\	& & Rate	&   $ 0.52 \pm 0.02 $	&   $ 0.54 \pm 0.02 $	&   $ 0.49 \pm 0.02 $	\\ \cmidrule(r){1-2}
\multirow{2}{*}{	VIB, $\beta$=\num{1e-2}	}	& \multirow{2}{*}{$ 34.6 \pm  0.4$}	& Entropy	&   $ 0.60 \pm 0.01 $	&   $ 0.62 \pm 0.01 $	&   $ 0.57 \pm 0.01 $	\\	& & Rate	&   $ 0.52 \pm 0.04 $	&   $ 0.55 \pm 0.04 $	&   $ 0.48 \pm 0.03 $	\\ \cmidrule(r){1-2}
\multirow{2}{*}{	VIB, $\beta$=\num{1e-3}	}	& \multirow{2}{*}{$ 35.6 \pm  0.5$}	& Entropy	&   $ 0.59 \pm 0.01 $	&   $ 0.60 \pm 0.01 $	&   $ 0.56 \pm 0.01 $	\\	& & Rate	&   $ 0.50 \pm 0.04 $	&   $ 0.53 \pm 0.03 $	&   $ 0.48 \pm 0.03 $	\\ \cmidrule(r){1-2}
\multirow{2}{*}{	VIB, $\beta$=\num{1e-1}, D	}	& \multirow{2}{*}{$ 29.0 \pm  0.6$}	& Entropy	&   $ 0.57 \pm 0.01 $	&   $ 0.60 \pm 0.01 $	&   $ 0.53 \pm 0.01 $	\\	& & Rate	&   $ 0.45 \pm 0.02 $	&   $ 0.48 \pm 0.02 $	&   $ 0.46 \pm 0.01 $	\\ \cmidrule(r){1-2}
\multirow{2}{*}{	VIB, $\beta$=\num{1e-2}, D	}	& \multirow{2}{*}{$ 32.5 \pm  0.9$}	& Entropy	&   $ 0.62 \pm 0.01 $	&   $ 0.63 \pm 0.02 $	&   $ 0.59 \pm 0.01 $	\\	& & Rate	&   $ 0.53 \pm 0.05 $	&   $ 0.56 \pm 0.04 $	&   $ 0.52 \pm 0.04 $	\\ \cmidrule(r){1-2}
\multirow{2}{*}{	VIB, $\beta$=\num{1e-3}, D	}	& \multirow{2}{*}{$ 34.5 \pm  1.0$}	& Entropy	&   $ 0.63 \pm 0.01 $	&   $ 0.64 \pm 0.02 $	&   $ 0.60 \pm 0.01 $	\\	& & Rate	&   $ 0.56 \pm 0.05 $	&   $ 0.57 \pm 0.03 $	&   $ 0.54 \pm 0.05 $	\\ \cmidrule(r){1-2}
\multirow{2}{*}{	MASS, $\beta$=\num{1e-2}	}	& \multirow{2}{*}{$ 29.6 \pm  0.4$}	& Entropy	&   $ 0.59 \pm 0.01 $	&   $ 0.61 \pm 0.01 $	&   $ 0.56 \pm 0.01 $	\\	& & $\max_i q_\phi(f_\theta(x) | y_i)$	&   $ 0.43 \pm 0.03 $	&   $ 0.48 \pm 0.03 $	&   $ 0.43 \pm 0.01 $	\\ \cmidrule(r){1-2}
\multirow{2}{*}{	MASS, $\beta$=\num{1e-3}	}	& \multirow{2}{*}{$ 32.7 \pm  0.8$}	& Entropy	&   $ 0.57 \pm 0.01 $	&   $ 0.59 \pm 0.02 $	&   $ 0.55 \pm 0.01 $	\\	& & $\max_i q_\phi(f_\theta(x) | y_i)$	&   $ 0.57 \pm 0.04 $	&   $ 0.59 \pm 0.04 $	&   $ 0.54 \pm 0.03 $	\\ \cmidrule(r){1-2}
\multirow{2}{*}{	MASS, $\beta$=\num{1e-4}	}	& \multirow{2}{*}{$ 34.0 \pm  0.3$}	& Entropy	&   $ 0.57 \pm 0.01 $	&   $ 0.57 \pm 0.01 $	&   $ 0.55 \pm 0.01 $	\\	& & $\max_i q_\phi(f_\theta(x) | y_i)$	&   $ 0.59 \pm 0.03 $	&   $ 0.58 \pm 0.03 $	&   $ 0.57 \pm 0.03 $	\\ \cmidrule(r){1-2}
\multirow{2}{*}{	MASS, $\beta$=\num{0}	}	& \multirow{2}{*}{$ 34.1 \pm  0.6$}	& Entropy	&   $ 0.57 \pm 0.01 $	&   $ 0.58 \pm 0.01 $	&   $ 0.55 \pm 0.00 $	\\	& & $\max_i q_\phi(f_\theta(x) | y_i)$	&   $ 0.61 \pm 0.03 $	&   $ 0.59 \pm 0.04 $	&   $ 0.59 \pm 0.04 $	\\ \cmidrule(r){1-2}
\multirow{2}{*}{	MASS, $\beta$=\num{1e-2}, D	}	& \multirow{2}{*}{$ 29.3 \pm  1.2$}	& Entropy	&   $ 0.62 \pm 0.02 $	&   $ \mathbf{0.64 \pm 0.03} $	&   $ 0.59 \pm 0.02 $	\\	& & $\max_i q_\phi(f_\theta(x) | y_i)$	&   $ 0.50 \pm 0.05 $	&   $ 0.54 \pm 0.05 $	&   $ 0.47 \pm 0.03 $	\\ \cmidrule(r){1-2}
\multirow{2}{*}{	MASS, $\beta$=\num{1e-3}, D	}	& \multirow{2}{*}{$ 31.5 \pm  0.6$}	& Entropy	&   $ 0.61 \pm 0.02 $	&   $ 0.62 \pm 0.03 $	&   $ 0.58 \pm 0.01 $	\\	& & $\max_i q_\phi(f_\theta(x) | y_i)$	&   $ 0.62 \pm 0.04 $	&   $ \mathbf{0.63 \pm 0.04} $	&   $ 0.58 \pm 0.04 $	\\ \cmidrule(r){1-2}
\multirow{2}{*}{	MASS, $\beta$=\num{1e-4}, D	}	& \multirow{2}{*}{$ 32.7 \pm  0.8$}	& Entropy	&   $ 0.61 \pm 0.02 $	&   $ 0.61 \pm 0.03 $	&   $ 0.59 \pm 0.01 $	\\	& & $\max_i q_\phi(f_\theta(x) | y_i)$	&   $ 0.65 \pm 0.04 $	&   $ \mathbf{0.63 \pm 0.04} $	&   $ 0.62 \pm 0.05 $	\\ \cmidrule(r){1-2}
\multirow{2}{*}{	MASS, $\beta$=\num{0}, D	}	& \multirow{2}{*}{$ 32.2 \pm  1.1$}	& Entropy	&   $ 0.63 \pm 0.01 $	&   $ 0.64 \pm 0.02 $	&   $ 0.61 \pm 0.01 $	\\	& & $\max_i q_\phi(f_\theta(x) | y_i)$	&   $ \mathbf{0.65 \pm 0.05} $	&   $ \mathbf{0.64 \pm 0.05} $	&   $ 0.62 \pm 0.06 $	\\ 

\bottomrule
\end{tabular}
\end{small}
\end{center}
\vskip -0.1in
\end{table*}

\begin{table*}
\caption{Out-of-distribution detection metrics for ResNet20 network trained on 40,000 CIFAR-10 images, with SVHN as the out-of-distribution examples.  Full experiment details are in Supplementary Material \ref{supp:ExperimentDetails}.  Values are the mean over 4 training runs with different random seeds, plus or minus the standard deviation. Emboldened values are those for which the maximum observed mean value in the column was within one standard deviation.  Higher values are better.}
\label{tab:ResNet-OOD-40k}
\vskip 0.15in
\begin{center}
\begin{small}
\begin{tabular}{@{}lc|cccc@{}}
\toprule
Training Method & Test Accuracy & Detection Method & AUROC & APR In & APR Out \\
\midrule

\multirow{2}{*}{	SoftmaxCE	}	& \multirow{2}{*}{$ 81.7 \pm  0.3$}	& Entropy	&   $ \mathbf{0.77 \pm 0.02} $	&   $ 0.81 \pm 0.02 $	&   $ 0.70 \pm 0.02 $	\\	& & $\max_i q_\phi(f_\theta(x) | y_i)$	&   $ 0.59 \pm 0.03 $	&   $ 0.62 \pm 0.03 $	&   $ 0.55 \pm 0.02 $	\\ \cmidrule(r){1-2}
\multirow{2}{*}{	VIB, $\beta$=\num{1e-3}	}	& \multirow{2}{*}{$ 81.0 \pm  0.3$}	& Entropy	&   $ 0.74 \pm 0.02 $	&   $ 0.79 \pm 0.02 $	&   $ 0.67 \pm 0.02 $	\\	& & Rate	&   $ 0.55 \pm 0.04 $	&   $ 0.57 \pm 0.05 $	&   $ 0.51 \pm 0.03 $	\\ \cmidrule(r){1-2}
\multirow{2}{*}{	VIB, $\beta$=\num{1e-4}	}	& \multirow{2}{*}{$ 81.2 \pm  0.4$}	& Entropy	&   $ 0.73 \pm 0.02 $	&   $ 0.76 \pm 0.03 $	&   $ 0.66 \pm 0.02 $	\\	& & Rate	&   $ 0.50 \pm 0.02 $	&   $ 0.54 \pm 0.02 $	&   $ 0.48 \pm 0.01 $	\\ \cmidrule(r){1-2}
\multirow{2}{*}{	VIB, $\beta$=\num{1e-5}	}	& \multirow{2}{*}{$ 80.9 \pm  0.5$}	& Entropy	&   $ 0.75 \pm 0.02 $	&   $ 0.80 \pm 0.02 $	&   $ 0.67 \pm 0.02 $	\\	& & Rate	&   $ 0.18 \pm 0.05 $	&   $ 0.34 \pm 0.01 $	&   $ 0.34 \pm 0.01 $	\\ \cmidrule(r){1-2}
\multirow{2}{*}{	VIB, $\beta$=\num{0}	}	& \multirow{2}{*}{$ 81.5 \pm  0.2$}	& Entropy	&   $ \mathbf{0.79 \pm 0.02} $	&   $ \mathbf{0.84 \pm 0.02} $	&   $ \mathbf{0.73 \pm 0.04} $	\\	& & Rate	&   $ 0.11 \pm 0.03 $	&   $ 0.32 \pm 0.01 $	&   $ 0.32 \pm 0.01 $	\\ \cmidrule(r){1-2}
\multirow{2}{*}{	MASS, $\beta$=\num{1e-3}	}	& \multirow{2}{*}{$ 75.8 \pm  0.5$}	& Entropy	&   $ 0.74 \pm 0.03 $	&   $ 0.77 \pm 0.03 $	&   $ 0.69 \pm 0.03 $	\\	& & $\max_i q_\phi(f_\theta(x) | y_i)$	&   $ 0.37 \pm 0.04 $	&   $ 0.43 \pm 0.02 $	&   $ 0.42 \pm 0.02 $	\\ \cmidrule(r){1-2}
\multirow{2}{*}{	MASS, $\beta$=\num{1e-4}	}	& \multirow{2}{*}{$ 80.6 \pm  0.5$}	& Entropy	&   $ \mathbf{0.76 \pm 0.04} $	&   $ \mathbf{0.80 \pm 0.04} $	&   $ \mathbf{0.70 \pm 0.05} $	\\	& & $\max_i q_\phi(f_\theta(x) | y_i)$	&   $ 0.48 \pm 0.06 $	&   $ 0.53 \pm 0.05 $	&   $ 0.47 \pm 0.04 $	\\ \cmidrule(r){1-2}
\multirow{2}{*}{	MASS, $\beta$=\num{1e-5}	}	& \multirow{2}{*}{$ 81.6 \pm  0.4$}	& Entropy	&   $ 0.77 \pm 0.01 $	&   $ \mathbf{0.82 \pm 0.01} $	&   $ \mathbf{0.71 \pm 0.02} $	\\	& & $\max_i q_\phi(f_\theta(x) | y_i)$	&   $ 0.54 \pm 0.03 $	&   $ 0.58 \pm 0.03 $	&   $ 0.51 \pm 0.02 $	\\ \cmidrule(r){1-2}
\multirow{2}{*}{	MASS, $\beta$=\num{0}	}	& \multirow{2}{*}{$ 81.5 \pm  0.2$}	& Entropy	&   $ \mathbf{0.79 \pm 0.03}$	&   $ \mathbf{0.83 \pm 0.02} $	&   $ \mathbf{0.73 \pm 0.03} $	\\	& & $\max_i q_\phi(f_\theta(x) | y_i)$	&   $ 0.49 \pm 0.04 $	&   $ 0.54 \pm 0.04 $	&   $ 0.47 \pm 0.02 $	\\

\bottomrule
\end{tabular}
\end{small}
\end{center}
\vskip -0.1in
\end{table*}

\begin{table*}
\caption{Out-of-distribution detection metrics for ResNet20 network trained on 10,000 CIFAR-10 images, with SVHN as the out-of-distribution examples.  Full experiment details are in Supplementary Material \ref{supp:ExperimentDetails}.  Values are the mean over 4 training runs with different random seeds, plus or minus the standard deviation. Emboldened values are those for which the maximum observed mean value in the column was within one standard deviation.  Higher values are better.}
\label{tab:ResNet-OOD-10k}
\vskip 0.15in
\begin{center}
\begin{small}
\begin{tabular}{@{}lc|cccc@{}}
\toprule
Training Method & Test Accuracy & Detection Method & AUROC & APR In & APR Out \\
\midrule

\multirow{2}{*}{	SoftmaxCE	}	& \multirow{2}{*}{$ 67.5 \pm  0.8$}	& Entropy	&   $ 0.64 \pm 0.02 $	&   $ 0.68 \pm 0.02 $	&   $ 0.58 \pm 0.02 $	\\	& & $\max_i q_\phi(f_\theta(x) | y_i)$	&   $ 0.59 \pm 0.03 $	&   $ 0.61 \pm 0.03 $	&   $ 0.57 \pm 0.04 $	\\ \cmidrule(r){1-2}
\multirow{2}{*}{	VIB, $\beta$=\num{1e-3}	}	& \multirow{2}{*}{$ 66.9 \pm  1.0$}	& Entropy	&   $ 0.59 \pm 0.02 $	&   $ 0.63 \pm 0.04 $	&   $ 0.54 \pm 0.02 $	\\	& & Rate	&   $ \mathbf{0.72 \pm 0.05} $	&   $ \mathbf{0.73 \pm 0.05} $	&   $ \mathbf{0.67 \pm 0.05} $	\\ \cmidrule(r){1-2}
\multirow{2}{*}{	VIB, $\beta$=\num{1e-4}	}	& \multirow{2}{*}{$ 66.4 \pm  0.5$}	& Entropy	&   $ 0.59 \pm 0.01 $	&   $ 0.63 \pm 0.02 $	&   $ 0.54 \pm 0.01 $	\\	& & Rate	&   $ 0.59 \pm 0.07 $	&   $ 0.60 \pm 0.07 $	&   $ 0.56 \pm 0.06 $	\\ \cmidrule(r){1-2}
\multirow{2}{*}{	VIB, $\beta$=\num{1e-5}	}	& \multirow{2}{*}{$ 67.9 \pm  0.8$}	& Entropy	&   $ 0.61 \pm 0.03 $	&   $ 0.65 \pm 0.04 $	&   $ 0.56 \pm 0.03 $	\\	& & Rate	&   $ 0.39 \pm 0.07 $	&   $ 0.42 \pm 0.03 $	&   $ 0.43 \pm 0.04 $	\\ \cmidrule(r){1-2}
\multirow{2}{*}{	VIB, $\beta$=\num{0}	}	& \multirow{2}{*}{$ 67.1 \pm  1.0$}	& Entropy	&   $ 0.64 \pm 0.01 $	&   $ 0.68 \pm 0.01 $	&   $ 0.58 \pm 0.01 $	\\	& & Rate	&   $ 0.32 \pm 0.03 $	&   $ 0.39 \pm 0.01 $	&   $ 0.39 \pm 0.01 $	\\ \cmidrule(r){1-2}
\multirow{2}{*}{	MASS, $\beta$=\num{1e-3}	}	& \multirow{2}{*}{$ 59.6 \pm  0.8$}	& Entropy	&   $ 0.59 \pm 0.02 $	&   $ 0.62 \pm 0.03 $	&   $ 0.56 \pm 0.02 $	\\	& & $\max_i q_\phi(f_\theta(x) | y_i)$	&   $ 0.49 \pm 0.07 $	&   $ 0.46 \pm 0.06 $	&   $ 0.48 \pm 0.08 $	\\ \cmidrule(r){1-2}
\multirow{2}{*}{	MASS, $\beta$=\num{1e-4}	}	& \multirow{2}{*}{$ 66.6 \pm  0.4$}	& Entropy	&   $ 0.62 \pm 0.02 $	&   $ 0.67 \pm 0.02 $	&   $ 0.56 \pm 0.03 $	\\	& & $\max_i q_\phi(f_\theta(x) | y_i)$	&   $ 0.61 \pm 0.05 $	&   $ 0.61 \pm 0.05 $	&   $ 0.60 \pm 0.05 $	\\ \cmidrule(r){1-2}
\multirow{2}{*}{	MASS, $\beta$=\num{1e-5}	}	& \multirow{2}{*}{$ 67.4 \pm  1.0$}	& Entropy	&   $ 0.64 \pm 0.02 $	&   $ 0.69 \pm 0.03 $	&   $ 0.58 \pm 0.01 $	\\	& & $\max_i q_\phi(f_\theta(x) | y_i)$	&   $ 0.61 \pm 0.08 $	&   $ 0.61 \pm 0.06 $	&   $ \mathbf{0.61 \pm 0.09} $	\\ \cmidrule(r){1-2}
\multirow{2}{*}{	MASS, $\beta$=\num{0}	}	& \multirow{2}{*}{$ 67.4 \pm  0.3$}	& Entropy	&   $ 0.64 \pm 0.01 $	&   $ 0.68 \pm 0.02 $	&   $ 0.58 \pm 0.01 $	\\	& & $\max_i q_\phi(f_\theta(x) | y_i)$	&   $ 0.55 \pm 0.05 $	&   $ 0.56 \pm 0.04 $	&   $ 0.54 \pm 0.05 $	\\

\bottomrule
\end{tabular}
\end{small}
\end{center}
\vskip -0.1in
\end{table*}

\begin{table*}
\caption{Out-of-distribution detection metrics for ResNet20 network trained on 2,500 CIFAR-10 images, with SVHN as the out-of-distribution examples.  Full experiment details are in Supplementary Material \ref{supp:ExperimentDetails}.  Values are the mean over 4 training runs with different random seeds, plus or minus the standard deviation. Emboldened values are those for which the maximum observed mean value in the column was within one standard deviation.  Higher values are better.}
\label{tab:ResNet-OOD-2.5k}
\vskip 0.15in
\begin{center}
\begin{small}
\begin{tabular}{@{}lc|cccc@{}}
\toprule
Training Method & Test Accuracy & Detection Method & AUROC & APR In & APR Out \\
\midrule

\multirow{2}{*}{	SoftmaxCE	}	& \multirow{2}{*}{$ 50.0 \pm  0.7$}	& Entropy	&   $ 0.51 \pm 0.01 $	&   $ 0.52 \pm 0.02 $	&   $ 0.49 \pm 0.01 $	\\	& & $\max_i q_\phi(f_\theta(x) | y_i)$	&   $ 0.63 \pm 0.04 $	&   $ 0.62 \pm 0.03 $	&   $ 0.63 \pm 0.04 $	\\ \cmidrule(r){1-2}
\multirow{2}{*}{	VIB, $\beta$=\num{1e-3}	}	& \multirow{2}{*}{$ 49.5 \pm  1.1$}	& Entropy	&   $ 0.48 \pm 0.05 $	&   $ 0.50 \pm 0.05 $	&   $ 0.47 \pm 0.03 $	\\	& & Rate	&   $ \mathbf{0.68 \pm 0.07} $	&   $ \mathbf{0.68 \pm 0.05} $	&   $ \mathbf{0.66 \pm 0.08} $	\\ \cmidrule(r){1-2}
\multirow{2}{*}{	VIB, $\beta$=\num{1e-4}	}	& \multirow{2}{*}{$ 49.4 \pm  1.0$}	& Entropy	&   $ 0.47 \pm 0.05 $	&   $ 0.50 \pm 0.05 $	&   $ 0.47 \pm 0.03 $	\\	& & Rate	&   $ \mathbf{0.66 \pm 0.09} $	&   $ \mathbf{0.65 \pm 0.08} $	&   $ \mathbf{0.66 \pm 0.09} $	\\ \cmidrule(r){1-2}
\multirow{2}{*}{	VIB, $\beta$=\num{1e-5}	}	& \multirow{2}{*}{$ 50.0 \pm  1.1$}	& Entropy	&   $ 0.48 \pm 0.05 $	&   $ 0.49 \pm 0.05 $	&   $ 0.48 \pm 0.03 $	\\	& & Rate	&   $ 0.59 \pm 0.10 $	&   $ 0.55 \pm 0.08 $	&   $ 0.61 \pm 0.09 $	\\ \cmidrule(r){1-2}
\multirow{2}{*}{	VIB, $\beta$=\num{0}	}	& \multirow{2}{*}{$ 50.6 \pm  0.8$}	& Entropy	&   $ 0.51 \pm 0.07 $	&   $ 0.54 \pm 0.08 $	&   $ 0.50 \pm 0.06 $	\\	& & Rate	&   $ \mathbf{0.52 \pm 0.20} $	&   $ 0.53 \pm 0.15 $	&   $ 0.56 \pm 0.17 $	\\ \cmidrule(r){1-2}
\multirow{2}{*}{	MASS, $\beta$=\num{1e-3}	}	& \multirow{2}{*}{$ 38.2 \pm  0.7$}	& Entropy	&   $ 0.48 \pm 0.04 $	&   $ 0.50 \pm 0.04 $	&   $ 0.47 \pm 0.03 $	\\	& & $\max_i q_\phi(f_\theta(x) | y_i)$	&   $ 0.54 \pm 0.11 $	&   $ 0.48 \pm 0.06 $	&   $ 0.51 \pm 0.08 $	\\ \cmidrule(r){1-2}
\multirow{2}{*}{	MASS, $\beta$=\num{1e-4}	}	& \multirow{2}{*}{$ 49.9 \pm  1.0$}	& Entropy	&   $ 0.49 \pm 0.04 $	&   $ 0.51 \pm 0.05 $	&   $ 0.48 \pm 0.03 $	\\	& & $\max_i q_\phi(f_\theta(x) | y_i)$	&   $ \mathbf{0.72 \pm 0.08} $	&   $ \mathbf{0.71 \pm 0.08} $	&   $ \mathbf{0.73 \pm 0.08} $	\\ \cmidrule(r){1-2}
\multirow{2}{*}{	MASS, $\beta$=\num{1e-5}	}	& \multirow{2}{*}{$ 50.1 \pm  0.5$}	& Entropy	&   $ 0.50 \pm 0.06 $	&   $ 0.51 \pm 0.06 $	&   $ 0.49 \pm 0.04 $	\\	& & $\max_i q_\phi(f_\theta(x) | y_i)$	&   $ \mathbf{0.69 \pm 0.10} $	&   $ \mathbf{0.68 \pm 0.10} $	&   $ \mathbf{0.70 \pm 0.10} $	\\ \cmidrule(r){1-2}
\multirow{2}{*}{	MASS, $\beta$=\num{0}	}	& \multirow{2}{*}{$ 50.2 \pm  1.0$}	& Entropy	&   $ 0.51 \pm 0.06 $	&   $ 0.53 \pm 0.06 $	&   $ 0.50 \pm 0.04 $	\\	& & $\max_i q_\phi(f_\theta(x) | y_i)$	&   $ \mathbf{0.69 \pm 0.07} $	&   $ \mathbf{0.68 \pm 0.07} $	&   $ \mathbf{0.68 \pm 0.07} $	\\

\bottomrule
\end{tabular}
\end{small}
\end{center}
\vskip -0.1in
\end{table*}

\subsection{Does MASS Learning finally solve the mystery of why stochastic gradient descent with the cross entropy loss works so well in deep learning?}
We do not believe so.  Figure \ref{fig:trainingcurves} shows how the values of the three terms in $\widehat{\mathcal{L}}_{MASS}$ change as the SmallMLP network trains on the CIFAR-10 dataset using either the SoftmaxCE training or MASS Learning.  Despite achieving similar accuracies, the SoftmaxCE training method does not seem to be implicitly performing MASS Learning, based on the differing values of the entropy (orange) and Jacobian (green) terms between the two methods as training progresses.

\begin{figure*}
\begin{center}
\centerline{\includegraphics[width=\textwidth]{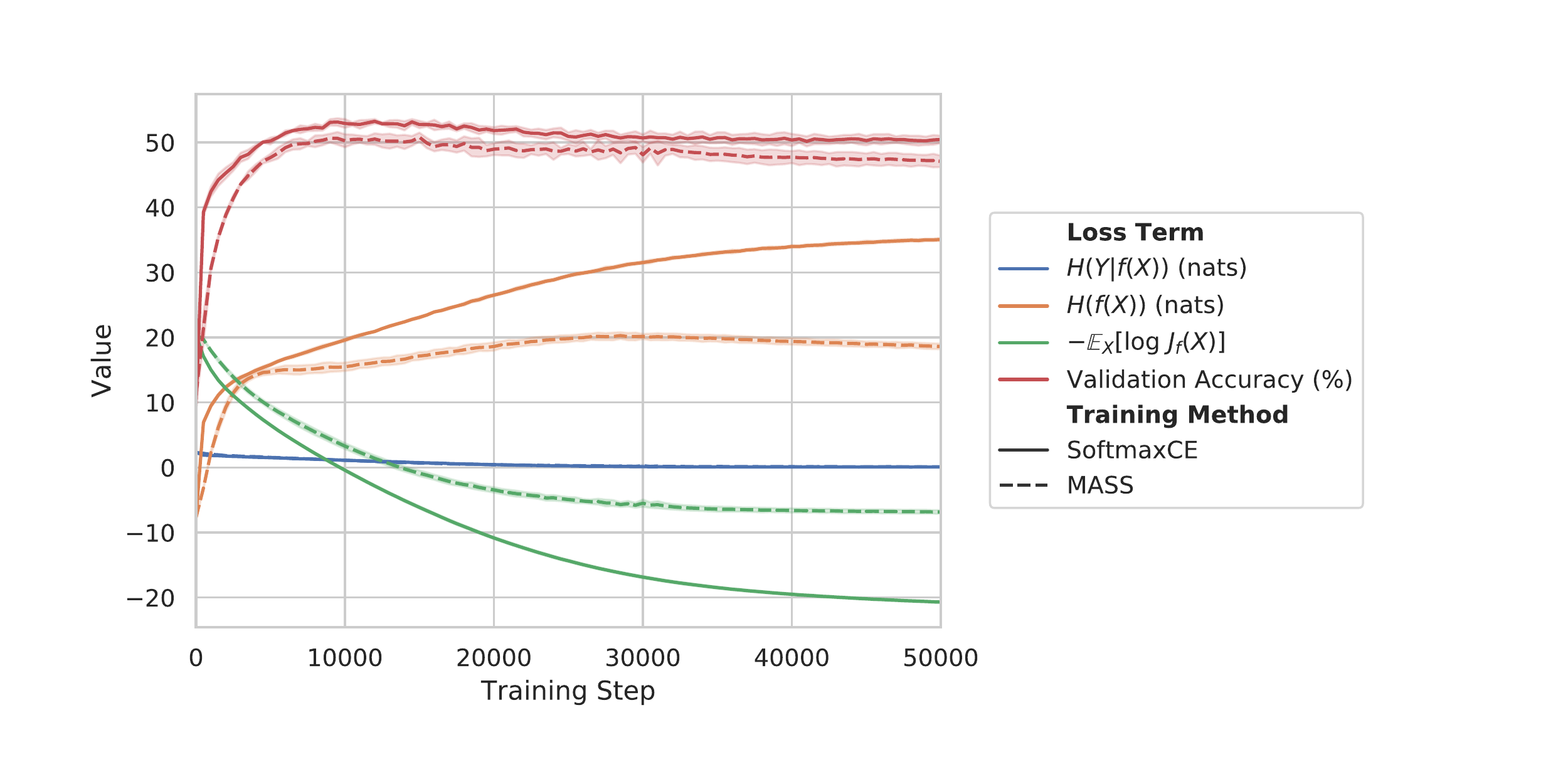}}
\caption{Estimated value of each term in the MASS Learning loss function, $\mathcal{L}_{MASS}(f)  = H(Y|f(X)) + \beta H(f(X)) - \beta \E_X[\log J_{f}(X)]$, during training of the SmallMLP network on the CIFAR-10 dataset.  The MASS training was performed with $\beta = 0.001$, though the plotted values are for the terms without being multiplied by the $\beta$ coefficients.  The values of these terms for SoftmaxCE training are estimated using a distribution $q_\phi(f_\theta(x) | y)$, with the distribution parameters $\phi$ being estimated at each training step by MLE over the training data.}
\label{fig:trainingcurves}
\end{center}
\end{figure*}

\section{Discussion}
MASS Learning is a new approach to representation learning that performs well on classification accuracy, regularization, and uncertainty quantification benchmarks, despite not being directly formulated for any of these tasks.  It shows particularly strong performance in improving uncertainty quantification.

There are several potential ways to improve MASS Learning.  Starting at the lowest level: it is likely that we did not manage to minimize $\widehat{\mathcal{L}}_{MASS}$ anywhere close to the extent possible in our experiments, given the minimal hyperparameter tuning we performed.  In particular, we noticed that the initialization of the variational distribution played a large role in performance, but we were not able to fully explore it.  

Moving a level higher, it may be that we are effectively minimized $\widehat{\mathcal{L}}_{MASS}$, but that $\widehat{\mathcal{L}}_{MASS}$ is not a useful empirical approximation or upper bound to $\mathcal{L}_{MASS}$.  This could be due to an insufficiently expressive variational distribution, or simply that the quantities in $\widehat{\mathcal{L}}_{MASS}$ require more data to approximate well than our datasets contained.

At higher levels still, it may be the case that the Lagrangian formulation of Theorem \ref{thm:MASS} as $\mathcal{L}_{MASS}$ is impractical for finding minimal achievable sufficient statistics.  Or it may be that the difference between minimal and minimal achievable sufficient statistics is relevant for performance on machine learning tasks.  Or it may simply be that framing machine learning as a problem of finding minimal sufficient statistics is not productive. 

Finally, while we again note that more work is needed to reduce the computational cost of our implementation of MASS Learning, we believe the concept of MASS learning, and the concepts of minimal achievability and Conserved Differential Information we introduce along with it, are beneficial to the theoretical understanding of representation learning.

\section*{Acknowledgements}

We would like to thank Georg Pichler, Thomas Vidick, Alex Alemi, Alessandro Achille, and Joseph Marino  for useful discussions.

\bibliography{main}
\bibliographystyle{icml2019}

\newpage
\onecolumn

\section{Supplementary Material}
\subsection{Standard Definition of Minimal Sufficient Statistics}
\label{supp:MSS}
The most common phrasing of the definition of \emph{minimal sufficient statistic} is:
\begin{definition}[Minimal Sufficient Statistic]
\label{def:tradMSS}
A sufficient statistic $f(X)$ for $Y$ is \emph{minimal} if for any other sufficient statistic $h(X)$ there exists a measurable function $g$ such that $f=g\circ h$ almost everywhere.
\end{definition}

Some references do not explicitly mention the ``measurability'' and ``almost everywhere'' conditions on $g$, but since we are in the probabilistic setting it is this definition of $f=g\circ h$ that is meaningful.

Our preferred phrasing of the definition of \emph{minimal sufficient statistic}, which we use in our Introduction, is:

\begin{definition}[Minimal Sufficient Statistic]
\label{def:betterMSS}
A sufficient statistic $f(X)$ for $Y$ is \emph{minimal} if for any measurable function $g$, $g(f(X))$ is no longer sufficient for $Y$ unless $g$ is invertible almost everywhere (i.e.  there exist a measurable function $g^{-1}$ and a set $\mathcal{A}$ such that $g^{-1}(g(x))=x$ for all $x \in \mathcal{A}$ and the event $\{ X \in \mathcal{A}^c\}$ has probability zero).
\end{definition}

The equivalence of Definition \ref{def:tradMSS} and Definition \ref{def:betterMSS} is given by the following lemma:

\begin{lemma}
Assume that there exists a minimal sufficient statistic $h(X)$ for $Y$ by Definition \ref{def:tradMSS}.
Then a sufficient statistic $f(X)$ is minimal in the sense of Definition \ref{def:tradMSS} if and only if it is minimal in the sense of Definition \ref{def:betterMSS}.
\end{lemma}
\begin{proof}
We first assume that $f(X)$ is minimal in the sense of Definition \ref{def:tradMSS}. 
Let $g$ be any measurable function such that $g(f(X))$ is sufficient for $Y$.
By the minimality (Def.~\ref{def:tradMSS}) of $f$ there must exist a measurable function $\tilde{g}$ such that $\tilde{g}(g(f(x)))=f(x)$ almost everywhere. 
This proves that $f$ is minimal  in the sense of Definition \ref{def:betterMSS}.

Now assume that $f(X)$ is minimal  in the sense of Definition \ref{def:betterMSS} and let $\tilde{f}(X)$ be another sufficient statistic.
Because $h$ is minimal (Def.~\ref{def:tradMSS}), there exist $g_1$ such that $h=g_1\circ \tilde{f}$ almost everywhere and $g_2$ such that $h=g_2\circ f$ almost everywhere.
Because $f$ is minimal  (Def.~\ref{def:betterMSS}), $g_2$ must be one-to-one almost everywhere, i.e. there exists a $\tilde{g}_2$ such that $\tilde{g}_2 \circ h = \tilde{g}_2 \circ g_2 \circ f=f$ almost everywhere.
In turn, we obtain that $\tilde{g}_2 \circ g_1\circ \tilde{f}=f$ almost everywhere, and since $\tilde{f}$ was arbitrary this proves the minimality of $f$  in the sense of Definition \ref{def:tradMSS}.
\end{proof}

\subsection{The Mutual Information Between the Input and Output of a Deep Network is Infinite}
\label{supp:MIInfinite}

Typically the mutual information between continuous random variables $X$ and $Y$ is given by 
\[
I(X,Y) = \int p(x,y) \log \frac{p(x,y)}{p(x)p(y)} \intd x \intd y,
\]
but this quantity is only defined when the joint density $p(x,y)$ is integrable, which it is not in the case that $Y = f(X)$.  (The technical term for $p(x, y)$ in this case is a ``singular distribution''.)  Instead, to compute $I(X, f(X))$ we must refer to the ``master definition'' of mutual information \citep{cover_elements_2006}, which is 
\be\label{MasterMIDef}
I(X,Y) = \sup_{\mathcal{P},\mathcal{Q}} I([X]_\mathcal{P}, [Y]_\mathcal{Q}),
\ee
where $\mathcal{P}$ and $\mathcal{Q}$ are finite partitions of the range of $X$ and $Y$, respectively, and $[X]_\mathcal{P}$ is the random variable obtained by quantizing $X$ using partition $\mathcal{P}$, and analogously for $[Y]_\mathcal{Q}$.

From this definition, we can prove the following Lemma:

\begin{lemma}\label{InfiniteMI}

If $X$ and $Y$ are continuous random variables, and there are open sets $O_X$ and $O_Y$ in the support of $X$ and $Y$, respectively, such that $y = f(x)$ for $x \in O_X$ and $y \in O_Y$, then $I(X,Y) = \infty$.  

This includes all $X$ and $Y$ where $Y = f(X)$ for an $f$ that is continuous somewhere on its domain, e.g., any deep network (considered as a function from an input vector to an output vector).
\end{lemma}
\begin{proof}

Suppose $X$ and $Y$ satisfy the conditions of the lemma.  
Let $O_X$ and $O_Y$ be open sets with $f(O_X) = O_Y$ and $\Prob[X\in O_X]=:\delta>0$, which exist by the lemma's assumptions.  
Then let $\mathcal{P}_{O_Y}^n$ be a partition of $O_Y$ into $n$ disjoint sets. 
Because $Y$ is continuous and hence does not have any atoms, we may assume that the probability of $Y$ belonging to each element of $\mathcal{P}_{O_Y}^n$ is equal to the same nonzero value $\delta/n$.  
Denote by $\mathcal{P}_{O_X}^n$ the partition of $O_X$ into $n$ disjoint sets, where each set in $\mathcal{P}_{O_X}^n$ is the preimage of one of the sets in $\mathcal{P}_{O_Y}^n$.  
We can construct  partitions of the whole domains of $X$ and $Y$ as $\mathcal{P}_{O_X}^n \cup O_X^c$ and $\mathcal{P}_{O_Y}^n \cup O_Y^c$, respectively.
Using these partitions in \eqref{MasterMIDef}, we obtain
\begin{align*}
I(X,Y)
& \geq (1-\delta) \log (1-\delta) + \sum_{A \in [X]_{\mathcal{P}_{O_X}^n}}  \Prob[X\in A, Y\in f(A)]\log \frac{\Prob[X\in A, Y\in f(A)]}{\Prob[X\in A]\Prob[Y\in f(A)]} \\
& = (1-\delta) \log (1-\delta) + n \frac{\delta}{n} \log \frac{\frac{\delta}{n}}{\frac{\delta}{n}\frac{\delta}{n}} \\
& = (1-\delta) \log (1-\delta) +  \delta \log \frac{n}{ \delta }.
\end{align*}
By letting $n$ go to infinity, we can see that the supremum in Eq.\ \ref{MasterMIDef} is infinity.
\end{proof}

\subsection{The Change of Variables Formula for Non-invertible Mappings}
\label{supp:coarea}

The change of variables formula is widely used in machine learning and is key to recent results in density estimation and generative modeling like normalizing flows \cite{rezende_variational_2015}, NICE \cite{dinh_nice:_2014}, and Real NVP \cite{dinh_density_2016}.  But all uses of the change of variables formula in the machine learning literature that we are aware of use it with respect to bijective mappings between random variables, despite the formula also being applicable to non-invertible mappings between random variables.
To address this gap, we offer the following brief tutorial.

The familiar form of the change of variables formula for a random variable $X$ with density $p(x)$ and a bijective, differentiable function $f \colon \R^d \to \R^d$ is
\be\label{eq:coarea1}
    \int_{\R^d} p(x) J_f(x)  \, \intd x = \int_{\R^d} p(f^{-1}(y)) \, \intd y.
\ee
where $J_f(x) = \big\lvert \det  \frac{\partial f(x)}{\partial x^\trans}\big\rvert$.

A slightly more general phrasing of Equation $\ref{eq:coarea1}$ is
\be\label{eq:coarea2}
    \int_{f^{-1}(\mathcal{B})} g(x) J_f(x)  \, \intd x = \int_{\mathcal{B}} g(f^{-1}(y)) \, \intd y.
\ee
where $g\colon \R^d \to \R$ is any non-negative measurable function, and $\mathcal{B}\subseteq \R^d$ is any measurable subset of $\R^d$.

We can extend Equation \ref{eq:coarea2} to work in the case that $f$ is not invertible.  To do this, we must address two issues.
First, if $f$ is not invertible, then $f^{-1}(y)$ is not a single point but rather a set.
Second, if $f$ is not invertible, then the Jacobian matrix $\frac{\partial f(x)}{\partial x^\trans}$ may not be square, and thus has no well defined determinant. 
Both issues can be resolved and lead to the following change of variables theorem \cite{krantz_geometric_2009}, which is based on the so-called coarea formula \cite{federer_geometric_1969}.
\begin{theorem}\label{thm:coareaCOV}
    Let $f \colon \R^d \to \R^r$ with $r\leq d$ be a differentiable function, $g\colon \R^d \to \R$  a non-negative measurable function, $\mathcal{B}\subseteq \R^d$ a measurable set,  and $J_f(x)=\sqrt{\det \left(\frac{\partial f(x)}{\partial x^\trans}\left(\frac{\partial f(x)}{\partial x^\trans}\right)^\trans\right)}$.
    Then
    \be\label{eq:coarea3}
        \int_{f^{-1}(\mathcal{B})} g(x) J_f(x)  \, \intd x = \int_{\mathcal{B}} \int_{f^{-1}(y)} g(x) \, \intd \Hm{d-r}(x) \, \intd y.
    \ee
where $\Hm{d-r}$ is the $(d-r)$-dimensional Hausdorff measure (one can think of this as a measure for lower-dimensional structures in high-dimensional space, e.g. the area of 2-dimensional surfaces in 3-dimensional space).\footnote{In what follows, we will sometimes replace $g$ by $g/J_f$ such that the Jacobian appears on the right-hand side. 
Furthermore, we will not only use non-negative $g$. 
This can be justified by splitting $g$ into positive and negative parts provided that either part results in a finite integral.}
\end{theorem}

We see in Theorem \ref{thm:coareaCOV} that Equation \ref{eq:coarea3} looks a lot like Equations \ref{eq:coarea1} and \ref{eq:coarea2}, but with $f^{-1}(y)$ replaced by an integral over the set $f^{-1}(y)$, which for almost every $y$ is a $(d-r)$-dimensional set.  And if $f$ in Equation \ref{eq:coarea3} happens to be bijective, Equation \ref{eq:coarea3} reduces to Equation \ref{eq:coarea2}.

We also see that the Jacobian determinant in Equation \ref{eq:coarea2} was replaced by the so-called $r$-dimensional Jacobian \[\sqrt{\det \left(\frac{\partial f(x)}{\partial x^\trans}\left(\frac{\partial f(x)}{\partial x^\trans}\right)^\trans\right)}\] in Equation \ref{eq:coarea3}. 
A word of caution is in order, as the $r$-dimensional Jacobian does not have the same nice properties for concatenated functions as does the Jacobian in the bijective case.  
In particular, we cannot calculate $J_{f_2 \circ f_1}$ based on the values of $J_{f_1}$ and $J_{f_2}$ because the product $\frac{\partial f_2(x)}{\partial x^\trans}\frac{\partial f_1(x)}{\partial x^\trans}\left(\frac{\partial f_2(x)}{\partial x^\trans}\frac{\partial f_1(x)}{\partial x^\trans}\right)^\trans$ does not decompose into a product of $\frac{\partial f_2(x)}{\partial x^\trans}\left(\frac{\partial f_2(x)}{\partial x^\trans}\right)^\trans$ and $\frac{\partial f_1(x)}{\partial x^\trans}\left(\frac{\partial f_1(x)}{\partial x^\trans}\right)^\trans$.
In other words, the trick used in techniques like normalizing flows and NICE to compute determinants of deep networks for use in the change of variables formula by decomposing the network's Jacobian into the product of layerwise Jacobians does not work straightforwardly in the case of non-invertible mappings.

\subsection{Motivation for Conserved Differential Information}
\label{supp:CDIDerivation}

First, we present an alternative definition of conditional entropy that is meaningful for singular distributions (e.g., the joint distribution $p(X, f(X))$ for a function $f$).  More information on this definition can be found in Koliander et al. \yrcite{koliander_entropy_2016}.

\subsubsection{Singular Conditional Entropy}
Assume that the random variable $X$ has a probability density function $p_X(x)$ on $\R^d$.
For a given differentiable function $f \colon \R^d \to \R^r$ ($r\leq d$), we want to analyze the conditional differential entropy $H(X|f(X))$. 
Following  Koliander et al. \yrcite{koliander_entropy_2016}, we define this quantity as:
\be\label{eq:condentropygen}
H(X|f(X))= - \int_{\R^r} p_{f(X)}(y) \int_{f^{-1}(y)} \theta^{d-r}_{\text{Pr}\{X\in \cdot|f(X)=y\}}(x) \log \big(\theta^{d-r}_{\text{Pr}\{X\in \cdot|f(X)=y\}}(x)\big) \, \intd\Hm{d-r}(x)\, \intd y
\ee
where $\Hm{d-r}$ denotes $(d-r)$-dimensional Hausdorff measure.
The function $p_{f(X)}$ is the probability density function of the random variable $f(X)$.
Although $\theta^{d-r}_{\text{Pr}\{X\in \cdot|f(X)=y\}}$ can also be interpreted as a probability density, it is not the commonly used density with respect to Lebesgue measure (which does not exist for $X|f(X)=y$) but a density with respect to a lower-dimensional Hausdorff measure.
We will analyze the two functions $p_{f(X)}$ and $\theta^{d-r}_{\text{Pr}\{X\in \cdot|f(X)=y\}}$ in more detail.
The density $p_{f(X)}$ is defined by the relation
\be
\int_{f^{-1}(\mathcal{B})} p_X(x) \, \intd x = \int_{\mathcal{B}}p_{f(X)}(y) \, \intd y\,,
\ee
which has to hold for every measurable set $\mathcal{B}\subseteq \R^r$.
Using the coarea formula (or the related change-of-variables theorem), we see that
\be
\int_{f^{-1}(\mathcal{B})} p_X(x) \, \intd x = \int_{\mathcal{B}}\int_{f^{-1}(y)} \frac{p_X(x)}{J_f(x)}\, \intd \Hm{d-r}(x) \, \intd y\,,
\ee
where $J_f(x)=\sqrt{\det \left(\frac{\partial f(x)}{\partial x^\trans}\left(\frac{\partial f(x)}{\partial x^\trans}\right)^\trans\right)}$ is the $r$-dimensional Jacobian determinant.
Thus, we identified
\be \label{eq:axdensitygen}
p_{f(X)}(y) = \int_{f^{-1}(y)} \frac{p_X(x)}{J_f(x)}\, \intd \Hm{d-r}(x)\,.
\ee

The second function, namely $\theta^{d-r}_{\text{Pr}\{X\in \cdot|f(X)=y\}}$, is the Radon-Nikodym derivative of the conditional probability $\text{Pr}\{X\in \cdot|f(X)=y\}$ with respect to $\Hm{d-r}$ restricted to the set where $X|f(X)=y$ has positive probability (in the end, this will be the set $f^{-1}(y)$). 
To understand this function, we have to know something about the conditional distribution of $X$ given $f(X)$.
Formally, a (regular) conditional probability $\text{Pr}\{X\in \cdot|f(X)=y\}$ has to satisfy three conditions:
\begin{itemize}
\item $\text{Pr}\{X\in \cdot|f(X)=y\}$ is a probability measure for each fixed $y\in \R^r$.
\item $\text{Pr}\{X\in \mathcal{A}|f(X)=\cdot\}$ is measurable for each fixed measurable set $\mathcal{A}\subseteq \R^d$.
\item For measurable sets $\mathcal{A}\subseteq \R^d$ and $\mathcal{B}\subseteq \R^r$, we have
\be\label{eq:condprobgen}
\text{Pr}\{(X,f(X))\in \mathcal{A}\times \mathcal{B}\} = \int_{\mathcal{B}} \text{Pr}\{X\in \mathcal{A}|f(X)=y\} p_{f(X)}(y) \, \intd y\,.
\ee
\end{itemize}
In our setting, \eqref{eq:condprobgen} becomes 
\be\label{eq:condprobourgen}
\int_{\mathcal{A}\cap f^{-1}(\mathcal{B})} p_X(x) \, \intd x= \int_{\mathcal{B}} \text{Pr}\{X\in \mathcal{A}|f(X)=y\} p_{f(X)}(y) \, \intd y\,.
\ee
Choosing 
\be
\text{Pr}\{X\in \mathcal{A}|f(X)=y\} = \frac{1}{p_{f(X)}(y)}\int_{\mathcal{A}\cap f^{-1}(y)} \frac{p_X(x)}{J_f(x)}\, \intd \Hm{d-r}(x)\,,
\ee
the right-hand side in \eqref{eq:condprobourgen} becomes 
\ba 
\int_{\mathcal{B}} \text{Pr}\{X\in \mathcal{A}|f(X)=y\} p_{f(X)}(y) \, \intd y
& = \int_{\mathcal{B}} \int_{\mathcal{A}\cap f^{-1}(y)} \frac{p_X(x)}{J_f(x)}\, \intd \Hm{d-r}(x) \, \intd y \notag \\
& = \int_{\mathcal{A}\cap f^{-1}(\mathcal{B})} p_X(x) \, \intd x\,,
\ea
where the final equality is again an application of the coarea formula.
Thus, we identified 
\be\label{eq:conddensitygen}
\theta^{d-r}_{\text{Pr}\{X\in \cdot|f(X)=y\}}(x) =  \frac{p_X(x)}{J_f(x) \, p_{f(X)}(y)}\,.
\ee

Although things might seem complicated up to this point, they simplify significantly once we put everything together.
In particular, inserting  \eqref{eq:conddensitygen} into \eqref{eq:condentropygen}, we obtain
\ba 
H(X|f(X)) 
& = -\int_{\R^r} p_{f(X)}(y) \int_{f^{-1}(y)} \frac{p_X(x)}{J_f(x) \, p_{f(X)}(y)} \log \bigg(\frac{p_X(x)}{J_f(x) \, p_{f(X)}(y)}\bigg) \, \intd\Hm{d-r}(x)\, \intd y  \notag \\
& = -\int_{\R^r} \int_{f^{-1}(y)} \frac{p_X(x)}{J_f(x) } \log \bigg(\frac{p_X(x)}{J_f(x) \, p_{f(X)}(y)}\bigg) \, \intd\Hm{d-r}(x)\, \intd y \notag \\
& = -\int_{\R^d}   p_X(x)  \log \bigg(\frac{p_X(x)}{J_f(x) \, p_{f(X)}(f(x))}\bigg)  \, \intd x \label{eq:applycoareagen} \\
& = H(X) + \int_{\R^d}   p_X(x)  \log \big( J_f(x)  p_{f(X)}(f(x))\big)  \, \intd x  \notag \\
& = H(X) + \int_{\R^d}   p_X(x)  \log \big( p_{f(X)}(f(x))\big)  \, \intd x 
 + \int_{\R^d}   p_X(x)  \log \big( J_f(x) \big)  \, \intd x \notag \\
& = H(X) + \int_{\R^r} \int_{f^{-1}(y)}  \frac{p_X(x)}{J_f(x) }  \log \big(p_{f(X)}(f(x))\big)  \, \intd\Hm{d-r}(x)\, \intd y 
+ \E\big[\log \big( J_f(X) \big)\big] \label{eq:applycoareagen2} \\
& =  H(X) + \int_{\R^r} \int_{f^{-1}(y)}  \frac{p_X(x)}{J_f(x) }   \, \intd\Hm{d-r}(x)  \log \big(p_{f(X)}(y)\big)\, \intd y
+  \E\big[\log \big( J_f(X) \big)\big]  \notag \\
& = H(X) + \int_{\R^r} p_{f(X)}(y)  \log \big(p_{f(X)}(y)\big) \, \intd y
+ \E\big[\log \big( J_f(X) \big)\big]  \notag \\
& = H(X) -H(f(X))
+  \E\big[\log \big( J_f(X) \big)\big] 
\label{eq:condentropy2gen}
\ea
where \eqref{eq:applycoareagen} and \eqref{eq:applycoareagen2} hold by the coarea formula.

So, altogether we have that for a random variable $X$ and a function $f$, the singular conditional entropy between $X$ and $f(X)$ is
\be
H(X|f(X)) = H(X) -H(f(X))
+  \E\big[\log \big( J_f(X) \big)\big].
\ee
This quantity can loosely be interpreted as being the difference in differential entropies between $X$ and $f(X)$ but with an additional term that corrects for any ``uninformative'' scaling that $f$ does.

\subsubsection{Conserved Differential Information}

For random variables that are not related by a deterministic function, mutual information can be expanded as
\be 
I(X,Y)= H(X)-H(X|Y)
\ee
where $H(X)$ and $H(X|Y)$ are differential entropy and conditional differential entropy, respectively.
As we would like to measure information between random variables that are deterministically dependent, we can mimic this behavior by defining for a Lipschitz continuous mapping $f$:
\be 
C(X, f(X)) := H(X)-H(X|f(X))\,.
\ee
By \eqref{eq:condentropy2gen}, this can be simplified to
\be \label{eq:singularMI}
C(X,f(X))= H(f(X)) -  \E\big[\log \big( J_f(X) \big)\big]  \,
\ee
yielding our definition of CDI.

\subsection{Proof of CDI Data Processing Inequality}
\label{supp:CDIProof}

\textbf{CDI Data Processing Inequality} (Theorem \ref{thm:dpi})

For Lipschitz continuous functions $f$ and $g$ with the same output space,
\[
C(X, f(X)) \geq C(X, g(f(X))
\]
with equality if and only if $g$ is one-to-one almost everywhere.

\begin{proof}
We calculate the difference between $C(X,f(X))$ and $C(X,g(f(X)))$.
\begin{align}
C(X,f(X))& -C(X,g(f(X)))\\
& = H(f(X)) - \E_X\big[\log J_f (X)\big] 
- H(g(f(X))) + \E_X\big[\log J_{g\circ f} (X)\big] 
\notag \\
& = H(f(X)) - H(g(f(X))) 
+ \E_X\bigg[\log \frac{J_g (f(X)) J_f (X)}{ J_f (X)}\bigg] \label{eq:diffgbij}\\
& = -\E_X[\log p_{f(X)}(f(X))]  
+ \E_X\bigg[\log \bigg( \sum_{z\in g^{-1}(g(f(X)))}\frac{p_{f(X)}(f(z))}{J_g(f(z))}\bigg)\bigg] 
+ \E_X [\log  J_g (f(X)) ] \label{eq:densofgf}\\
& = \E_X\left[ \log \left( \frac{\sum_{z\in g^{-1}(g(f(X)))} \frac{p_{f(X)}(f(z))}{J_g(f(z))} }{   \frac{p_{f(X)}(f(X))}{J_g(f(X))} } \right) \right]
\label{eq:dataproc}
\end{align}
where \eqref{eq:diffgbij} holds because the Jacobian determinant $J_{g\circ f}$ can be decomposed as $g$ has the same domain and codomain and \eqref{eq:densofgf} holds because the probability density function of $g(f(X))$ can be calculated as $p_{g(f(X))}(z)= \sum_{z\in g^{-1}(g(f(X)))}\frac{p_{f(X)}(f(z))}{J_g(f(z))}$ using a change of variables argument.
The resulting term in \eqref{eq:dataproc} is clearly always nonnegative which proves the inequality.

To prove the equality statement, we first assume that \eqref{eq:dataproc} is zero.
In this case, $\sum_{z\in g^{-1}(g(f(x)))} \frac{p_{f(X)}(f(z))}{J_g(f(z))}=\frac{p_{f(X)}(f(x))}{J_g(f(x))}$ almost everywhere.
Of course, we also have that $p_{f(X)}(f(x))>0$ almost everywhere.
Thus, there exists a set $\mathcal{A}$ of probability one such that $\sum_{z\in g^{-1}(g(f(x)))} \frac{p_{f(X)}(f(z))}{J_g(f(z))}=\frac{p_{f(X)}(f(x))}{J_g(f(x))}$ and $p_{f(X)}(f(x))>0$ for all $x\in \mathcal{A}$. 
In particular, the set $g^{-1}(g(f(x)))\cap \mathcal{A} =\{f(x)\}$ and hence $g$ is one-to-one almost everywhere.

For the other direction, assume that there exists $\tilde{g}$ such that $\tilde{g}(g(f(x)))=f(x)$ almost everywhere.
We can assume without loss of generality that $p_{f(X)}(f(x))=0$ for all $x$ that do not satisfy this equation. 
Restricting the expectation in \eqref{eq:dataproc} to the values that satisfy $\tilde{g}(g(f(x)))=f(x)$ does not change the expectation and gives the value zero. 
\end{proof}

\subsection{Theorem \ref{thm:DiscreteIBMSS} Only Holds in the Reverse Direction for Continuous $X$}
\label{supp:IBReverseDirection}

The specific claim we are making is as follows:
\begin{theorem}
Let $X$ be a continuous random variable drawn according to a distribution $p(X|Y)$ determined by the discrete random variable $Y$.  
Let $\mathcal{F}$ be the set of measurable functions of $X$ to any target space.  
If $f(X)$ is a minimal sufficient statistic of $X$ for $Y$ then
\begin{align}
f \in &\arg\min_{S \in \mathcal{F}}  \ I(X,S(X)) \notag\\
&s.t. \ \ I(S(X), Y) = \max_{S' \in \mathcal{F}} I(S'(X), Y).
\label{eq:notsufcond}
\end{align}
However, there may exist a function $f$ satisfying \eqref{eq:notsufcond} such that $f(X)$ is not a minimal sufficient statistic.
\end{theorem}
\begin{proof}
First, we prove the forward direction. According to Lemma~\ref{lem:SuffStat},
 $Z = f(X)$ is a sufficient statistic for $Y$ if and only if $I(Z, Y) = I(X, Y) = \max_{S'} I(S'(X), Y)$.
 To show the minimality condition in \eqref{eq:notsufcond} for a minimal sufficient statistic, assume that there exists $S(X)$ such that $I(S(X), Y) = \max_{S' \in \mathcal{F}} I(S'(X), Y)$ and $I(X,S(X))< I(X,f(X))$.
 Because $f$ is assumed to be a minimal sufficient statistic, there exists $g$ such that $f(X)=g(S(X))$ and by the data-processing inequality $I(X,S(X))\geq I(X,f(X))$, a contradiction.

Next, we give an example of a function satisfying \eqref{eq:notsufcond} such that $f(X)$ is not a minimal sufficient statistic.  The example is the case when $I(X, f(X))$ is not finite, as is the case when $f$ is a deterministic function and $X$ is continuous.  (See Lemma \ref{InfiniteMI}.) 
In this case, $I(X, S(X))$ is infinite for all deterministic, sufficient statistics $S$.  Thus the set $\arg \min_S I(X, S(X))$ contains not only the minimal sufficient statistics, but all deterministic sufficient statistics.  As a concrete example, consider two i.i.d.\ normally-distributed random variables with mean $\mu$: $X = (X_1, X_2) \sim \mathcal{N}(\mu,1)$.  $T(X) = \frac{X_1 + X_2}{2}$ is a minimal sufficient statistic for $\mu$.  $T'(X) = (\frac{X_1 + X_2}{2}, X_1 \cdot X_2)$ is a non-minimal sufficient statistic for $\mu$.  
However,  both statistics satisfy $T, T' \in \arg\min_{S \in \mathcal{F}}  \ I(X,S(X))$ since $\min_{S \in \mathcal{F}} I(X,S(X))=\infty$ under the constraint $I(S(X), Y) = \max_{S' \in \mathcal{F}} I(S'(X), Y)$.
\end{proof}

\subsection{Experiment Details}
\label{supp:ExperimentDetails}
Code to reproduce all experiments is available online at \url{https://github.com/mwcvitkovic/MASS-Learning}.

\subsubsection{Data}
In all experiments above, the models were trained on the CIFAR-10 dataset \cite{Krizhevsky2009LearningML}.  In the out-of-distribution detection experiments, the SVHN dataset \cite{Netzer2011ReadingDI} was used as the out-of-distribution dataset.  All channels in all datapoints were normalized to have zero mean and unit variance across their dataset.  No data augmentation was used in any experiments.

\subsubsection{Networks}
The SmallMLP network is a 2-hidden-layer, fully-connected network with \texttt{elu} nonlinearities \cite{clevert_fast_2015}.  The first hidden layer contains 400 hidden units; the second contains 200 hidden units.  Batch norm was applied after the linear mapping and before the nonlinearity of each hidden layer.  Dropout, when used, was applied after the nonlinearity of each hidden layer.  When used in VIB and MASS, the representation $f_\theta(x)$ was in $\R^{15}$, with the VIB encoder outputting parameters for a fully-covariant Gaussian distribution in $\R^{15}$. The marginal distribution in VIB and each component of the variational distribution $q_\phi$ (one component for each possible output class) in MASS were both mixtures of 10 full-covariance, 15-dimensional multivariate Gaussians.

The ResNet20 network is the 20-layer residual net of He et al. \yrcite{he_deep_2015}.  We adapted our implementation from \url{https://github.com/akamaster/pytorch_resnet_cifar10}, to whose authors we are very grateful.  When used in VIB and MASS, the representation $f_\theta(x)$ was in $\R^{20}$, with the VIB encoder outputting parameters for a diagonally-covariant Gaussian distribution in $\R^{20}$.  The marginal distribution in VIB and each component of the the variational distribution $q_\phi$ (one component for each possible output class) in MASS were both mixtures of 10 full-covariance, 20-dimensional multivariate Gaussians.

In experiments where a distribution $q_\phi(f_\theta(x) | y)$ is used in conjunction with a function $f_\theta$ trained by SoftmaxCE, each component of $q_\phi(f_\theta(x) | y)$ was a mixture of 10 full-covariance, 10-dimensional multivariate Gaussians, the parameters $\phi$ of which were estimated by MLE on the training set.

\subsubsection{Training}

The SmallMLP network in all experiments and with all training methods was trained using the Adam optimizer \cite{kingma_adam:_2014} with a learning rate of $0.0005$ for 100,000 steps of stochastic gradient descent, using minibatches of size 256.  All quantities we report in this paper were fully-converged to stable values by 100,000 steps.  When training VIB, 5 encoder samples per datapoint were used during training, and 10 during testing.  When training MASS, the learning rate of the parameters of the variational distribution $q_\phi$ was set at \num{2.5e-5} to aid numerical stability.

The ResNet20 network in all experiments and with all training methods was trained using SGD with an initial learning rate of $0.1$, decayed by a multiplicative factor of 0.1 at epochs 100 and 150, a momentum factor of 0.9, and minibatches of size 128.  These values were taken directly from the original paper \cite{he_deep_2015}.  However, unlike the original paper, we did not use data augmentation in order to keep the comparison between different numbers of training points more rigorous.  This, combined with the smaller number of training points used, accounts for the around 82\% accuracy we observe on CIFAR-10 compared to the around 91\% accuracy in the original paper.  We trained the network for 70,000 steps of stochastic gradient descent.  All quantities we report in this paper were fully-converged to stable values by 70,000 steps.  When training VIB, 10 encoder samples per datapoint were used during training, and 20 during testing.  When training MASS, the learning rate of the parameters of the variational distribution was the same as those of the network.

The values of $\beta$ we chose for VIB and MASS were selected so that the largest $\beta$ value used in each experiment was much larger in magnitude than the remaining terms in the VIB or MASS training loss, and the smallest $\beta$ value used was much smaller than the remaining terms.  We made this choice in the hope of clearly observing the effect of the $\beta$ parameter and more fairly comparing SoftmaxCE, VIB, and MASS.  But we note that a finer-tuning of the $\beta$ parameter would likely result in better performance for both VIB and MASS.  We also note that the reason we omit a $\beta=0$ run for VIB with the SmallMLP network was that we could not prevent training from failing due to numerical instability with $\beta=0$ with this network.

\end{document}